\documentclass[twoside]{article}

\title{\sc Representing Independence Models with Elementary Triplets}
\author{{\small \bf Jose M. Pe\~{n}a}\\
\small IDA, Link\"oping University, Sweden\\ \small jose.m.pena@liu.se}
\date{}
\usepackage{graphicx,amssymb,datetime,float,MnSymbol,currfile,amsthm}
\usepackage{tikz}
\usetikzlibrary{arrows}
\usetikzlibrary{decorations.markings}

\usepackage{framed}

\newtheorem{theorem}{Theorem}
\newtheorem{lemma}{Lemma}

\newtheorem{example}{Example}

\def\tci{\tilde{\perp}}
\def\tp{\tilde{p}}
\def\ci{\!\perp\!}
\def\nci{\!\not\perp\!}
\def\ra{\rightarrow}
\def\la{\leftarrow}

\newcommand{\comments}[1]{}

\newcommand{\om}[1]{\overline{\mathcal{#1}}}
\newcommand{\m}[1]{\mathcal{#1}}

\tikzset{tt/.style={decoration={
  markings,
  mark=at position .485 with {\arrow{>}},
  mark=at position .515 with {\arrow{<}}},postaction={decorate}}}

\begin{document}
\maketitle

\begin{abstract}
In an independence model, the triplets that represent conditional independences between singletons are called elementary. It is known that the elementary triplets represent the independence model unambiguously under some conditions. In this paper, we show how this representation helps performing some operations with independence models, such as finding the dominant triplets or a minimal independence map of an independence model, or computing the union or intersection of a pair of independence models, or performing causal reasoning. For the latter, we rephrase in terms of conditional independences some of Pearl's results for computing causal effects.
\end{abstract}

\section{Introduction}\label{sec:introduction}

In this paper, we explore a non-graphical approach to representing and reasoning with independence models. The approach consists in representing an independence model by its elementary triplets, i.e. the triplets that represent conditional independences between individual random variables. It is known that the elementary triplets represent the independence model unambiguously when the independence model satisfies the semi-graphoid properties \cite{Anetal.1992,Matus1992,Studeny2005}. Moreover, every elementary triplet corresponds to an elementary imset, i.e. a function over the power set of the set of random variables at hand \cite{Studeny2005}. This provides an interesting connection between the question addressed in this paper and imset theory. Specifically, structural imsets are an algebraic method to represent independence models that solve some of the drawbacks of graphical models. Interestingly, every structural imset can be expressed as a linear combination of elementary imsets. For a detailed account of imset theory, we refer the reader to \cite{Studeny2005}. See also \cite{BouckaertHLS10} for a study about efficient ways of solving the implication problem between two structural imsets, i.e. deciding whether the independence model represented by one of the imsets is included in the model represented by the other. This paper aims to show how to reason efficiently with independence models when these are represented by elementary triplets, instead of by structural imsets. Another set of distinguished triplets that has been used in the literature to represent and reason with independence models are dominant triplets, i.e. any triplet that cannot be derived from any other triplet \cite{BaiolettiBV09,Baioletti20112,BaiolettiPV13,deWaal:2004:SIC:1036843.1036857,LopatatzidisvanderGaag,studeny1998complexity}. We will later briefly compare the relative merits of elementary and dominant triplets. We will also show how to produce the dominant triplets from the elementary triplets.

The rest of the paper is organized as follows. In Section \ref{sec:preliminaries}, we introduce some notation and concepts. In Section \ref{sec:representation}, we study under which conditions an independence model can unambiguously be represented by its elementary triplets. In Section \ref{sec:operations}, we show how this representation helps performing some operations with independence models, such as finding the dominant triplets or a minimal independence map of an independence model, or computing the union or intersection of a pair of independence models, or performing causal reasoning. Finally, we close the paper with some discussion in Section \ref{sec:discussion}.

\section{Preliminaries}\label{sec:preliminaries}

In this section, we introduce some notation and concepts. Let $V$ denote a finite set of random variables. Subsets of $V$ are denoted by upper-case letters, whereas elements of $V$ are denoted by lower-case letters. We shall not distinguish between elements of $V$ and singletons. Given two sets $I, J \subseteq V$, we use $I J$ to denote $I \cup J$. Union has higher priority than set difference in expressions. 

Given three disjoint sets $I, J, K \subseteq V$, the triplet $I \ci J | K$ denotes that $I$ is conditionally independent of $J$ given $K$. Given a set of triplets ${\m{M}}$, also known as an independence model, $I \ci_{\m{M}} J | K$ denotes that $I \ci J | K$ is in ${\m{M}}$ whereas $I \nci_{\m{M}} J | K$ denotes that $I \ci_{\m{M}} J | K$ does not hold. A triplet $I \ci J | K$ is called elementary if $|I| = |J| = 1$. Moreover, a triplet $I \ci J | K$ dominates another triplet $I' \ci J' | K'$ if $I' \subseteq I$, $J' \subseteq J$ and $K \subseteq K' \subseteq (I \setminus I') (J \setminus J') K$. Given a set of triplets, a triplet in the set is called dominant if no other triplet in the set dominates it.

Given a probability distribution $p(V)$ and three disjoint sets $I, J, K \subseteq V$, the triplet $I \ci_p J | K$ denotes that $I$ is conditionally independent of $J$ given $K$ in $p(V)$, i.e.
\[
p(I | J K) = p(I | J) \text{ whenever } p(J K ) > 0.
\]
The set of all such triplets is called the independence model induced by $p(V)$. Moreover, if $I \ci_p J | K L$ does not hold but 
\[
p(I | J K, L=l) = p(I | K, L=l) \text{ whenever } p(J K, L=l)>0
\]
where $l$ is a value in the domain of $L$, then we say that $I$ is conditionally independent of $J$ given $K$ and the context $l$ in $p(V)$, and we denote it by $I \ci_p J | K, L=l$.

Consider the following properties between triplets:

\begin{itemize}

\item[(CI0)] $I \ci J | K \Leftrightarrow J \ci I | K$.

\item[(CI1)] $I \ci J | K L, I \ci K | L \Leftrightarrow I \ci J K | L$.

\item[(CI2)] $I \ci J | K L, I \ci K | J L \Rightarrow I \ci J | L, I \ci K | L$.

\item[(CI3)] $I \ci J | K L, I \ci K | J L \Leftarrow I \ci J | L, I \ci K | L$.

\end{itemize}

A set of triplets with the properties CI0-1/CI0-2/CI0-3 is also called a semigraphoid/graphoid/compositional graphoid. For instance, the independence model induced by a probability distribution is a semigraphoid, while the independence model induced by a strictly positive probability distribution is a graphoid, and the independence model induced by a regular Gaussian distribution is a compositional graphoid. The CI0 property is also called symmetry property. The $\Rightarrow$ part of the CI1 property is also called contraction property, and the $\Leftarrow$ part corresponds to the so-called weak union and decomposition properties. The CI2 and CI3 properties are also called intersection and composition properties. Intersection is typically defined as $I \ci J | K L, I \ci K | J L \Rightarrow I \ci J K | L$. Note however that this and our definition are equivalent if CI1 holds. First, $I \ci J K | L$ implies $I \ci J | L$ and $I \ci K | L$ by CI1. Second, $I \ci J | L$ together with $I \ci K | J L$ imply $I \ci J K | L$ by CI1. Likewise, composition is typically defined as $I \ci J K | L \Leftarrow I \ci J | L, I \ci K | L$. Again, this and our definition are equivalent if CI1 holds. First, $I \ci J K | L$ implies $I \ci J | K L$ and $I \ci K | J L$ by CI1. Second, $I \ci K | J L$ together with $I \ci J | L$ imply $I \ci J K | L$ by CI1. In this paper, we will study sets of triplets that satisfy CI0-1, CI0-2 or CI0-3. So, the standard and our definitions are equivalent.

Consider also the following properties between elementary triplets:

\begin{itemize}

\item[(ci0)] $i \ci j | K \Leftrightarrow j \ci i | K$.

\item[(ci1)] $i \ci j | k L, i \ci k | L \Leftrightarrow i \ci k | j L, i \ci j | L$.

\item[(ci2)] $i \ci j | k L, i \ci k | j L \Rightarrow i \ci j | L, i \ci k | L$.

\item[(ci3)] $i \ci j | k L, i \ci k | j L \Leftarrow i \ci j | L, i \ci k | L$.

\end{itemize}
Note that CI2 and CI3 only differ in the direction of the implication. The same holds for ci2 and ci3. Note that ci0-3 are the elementary versions of CI0-3 with the only exception of ci1 and CI1.

Given a set of triplets $\m{M} = \{ I \ci J | K\}$, let
\[
\om{E} = e(\m{M}) = \{ i \ci j | M : I \ci_{\m{M}} J | K \text{ with } i \in I, j \in J \text{ and } K \subseteq M \subseteq (I \setminus i) (J \setminus j) K \}.
\]
Similarly, given a set of elementary triplets $\m{E} = \{ i \ci j | K\}$, let
\[
\om{M} = m(\m{E}) = \{ I \ci J | K : i \ci_{\m{E}} j | M \text{ for all } i \in I, j \in J \text{ and } K \subseteq M \subseteq (I \setminus i) (J \setminus j) K \}.
\]

We say that a set of triplets is closed under CI0-1/CI0-2/CI0-3 if applying the properties CI0-1/CI0-2/CI0-3 to triplets in the set always returns triplets that are in the set. Given a set of triplets $\m{M}$, we define its closure under CI0-1/CI0-2/CI0-3, denoted as $\m{M}^*$, as the minimal superset of $\m{M}$ that is closed under CI0-1/CI0-2/CI0-3. We define similarly the closure of a set of elementary triplets $\m{E}$ under ci0-1/ci0-2/ci0-3, which we denote as $\m{E}^*$.

Graphs can be used to represent independence models as follows. A directed and acyclic graph (DAG) is a graph that only has directed edges and does not have any subgraph of the form $i_1 \ra \ldots \ra i_n \ra i_1$. Given a DAG $G$ over $V$, a path between a node $i_1$ and a node $i_{n}$ on $G$ is a sequence of distinct nodes $i_{1}, \ldots, i_{n}$ such that $G$ has an edge between every pair of consecutive nodes in the sequence. If every edge in the path is of the form $i_j \ra i_{j+1}$, then $i_1$ is called an ancestor of $i_n$. Let $An_G(K)$ with $K \subseteq V$ denote the union of the ancestors of each node in $K$. A node $k$ on a path in $G$ is said to be a collider on the path if $i \ra k \la j$ is a subpath. Moreover, the path is said to be connecting given $K$ when
\begin{itemize}
\item every collider on the path is in $K \cup An_G(K)$, and

\item every non-collider on the path is outside $K$.
\end{itemize}
Let $I$, $J$ and $K$ denote three disjoint subsets of $V$. When there is no path in $G$ connecting a node in $I$ and a node in $J$ given $K$, we say that $I$ and $J$ are {\em d}-separated given $K$ in $G$, denoted as $I \ci_G J | K$. The independence model induced by $G$ consists of the triplets $I \ci J | K$ such that $I \ci_G J | K$.

We say that a DAG $G$ over $V$ is a minimal independence map of a set of triplets $\m{M}$ relative to an ordering $\sigma$ of the elements in $V$ if (i) $I \ci_G J | K$ implies that $I \ci_{\m{M}} J | K$, (ii) removing any edge from $G$ makes it cease to satisfy condition (i), and (iii) the edges of $G$ are of the form $\sigma(s) \ra \sigma(t)$ with $s < t$. Moreover, if $\m{M}$ is the independence model induced by a probability distribution $p(V)$, then the following factorization holds: 
\[
p(V) = \prod_{s=1}^{|V|} p(\sigma(s) | Pa_G(\sigma(s)))
\]
where $Pa_G(j) = \{i | i \ra j \text{ is in } G \}$ are the parents of $j$ in $G$. Moreover, $G$ is a perfect map of $\m{M}$ if $I \ci_G J | K$ implies $I \ci_{\m{M}} J | K$ and vice versa.

Finally, given three disjoint sets $X, Y, W \subseteq V$, we define the causal effect on $Y$ given $W$ of an intervention on $X$ as the conditional probability distribution of $Y$ given $W$ after setting $X$ to some value in its domain through an intervention, as opposed to an observation. We say that the causal effect is identifiable if it can be computed from observed quantities, i.e. from the probability distribution over $V$.

\section{Representation}\label{sec:representation}

In this section, we study the use of elementary triplets to represent independence models. We start by proving in the following lemma that there is a bijection between certain sets of triplets and certain sets of elementary triplets. The lemma has previously been proven when the sets of triplets and elementary triplets satisfy CI0-1 and ci0-1 \cite[Proposition 1]{Matus1992}. We extend it to the cases where they satisfy CI0-2/CI0-3 and ci0-2/ci0-3.

\begin{lemma}\label{lem:G2P}
If a set of triplets $\m{M}$ satisfies CI0-1/CI0-2/CI0-3 then $\om{E}$ satisfies ci0-1/ci0-2/ci0-3, $\m{M} = m(\om{E})$, and $\om{E} = \{ i \ci j | K : i \ci_{\m{M}} j | K \}$. Similarly, if a set of elementary triplets $\m{E}$ satisfies ci0-1/ci0-2/ci0-3 then $\om{M}$ satisfies CI0-1/CI0-2/CI0-3, $\m{E} = e(\om{M})$, and $\m{E} = \{ i \ci j | K : i \ci_{\om{M}} j | K \}$.
\end{lemma}

\begin{proof}
The lemma has previously been proven when $\m{M}$ and $\m{E}$ satisfy CI0-1 and ci0-1 \cite[Proposition 1]{Matus1992}. Therefore, we only have to prove that if $\m{M}$ satisfies CI0-3 then $\om{E}$ satisfies ci2-3, and that if $\m{E}$ satisfies ci0-3 then $\om{M}$ satisfies CI2-3.

\underline{Proof of CI0-2 $\Rightarrow$ ci2}

Assume that $i \ci_{\om{E}} j | k L$ and $i \ci_{\om{E}} k | j L$. Then, it follows from the definition of $\om{E}$ that $i \ci_{\m{M}} j | k L$ or $I \ci_{\m{M}} J | M$ with $i \in I$, $j \in J$ and $M \subseteq k L \subseteq (I \setminus i) (J \setminus j) M$. Note that the latter case implies that $i \ci_{\m{M}} j | k L$ by CI1. Similarly, $i \ci_{\om{E}} k | j L$ implies $i \ci_{\m{M}} k | j L$. Then, $i \ci_{\m{M}} j | L$ and $i \ci_{\m{M}} k | L$ by CI2. Then, $i \ci_{\om{E}} j | L$ and $i \ci_{\om{E}} k | L$ by definition of ${\om{E}}$.

\underline{Proof of CI0-3 $\Rightarrow$ ci3}

Assume that $i \ci_{\om{E}} j | L$ and $i \ci_{\om{E}} k | L$. Then, $i \ci_{\m{M}} j | L$ and $i \ci_{\m{M}} k | L$ by the same reasoning as before, which imply $i \ci_{\m{M}} j | k L$ and $i \ci_{\m{M}} k | j L$ by CI3. Then, $i \ci_{\om{E}} j | k L$ and $i \ci_{\om{E}} k | j L$ by definition of ${\om{E}}$.

\underline{Proof of ci0-2 $\Rightarrow$ CI2}

\begin{enumerate}

\item Assume that $I \ci_{\m{M}} j | k L$ and $I \ci_{\m{M}} k | j L$.\label{eq:2.14}

\item $i \ci_{\m{E}} j | k M$ and $i \ci_{\m{E}} k | j M$ for all $i \in I$ and $L \subseteq M \subseteq (I \setminus i) L$ follows from (\ref{eq:2.14}) by definition of ${\m{M}}$.\label{eq:2.16}

\item $i \ci_{\m{E}} j | M$ and $i \ci_{\m{E}} k | M$ for all $i \in I$ and $L \subseteq M \subseteq (I \setminus i) L$ by ci2 on (\ref{eq:2.16}).\label{eq:2.17}

\item $I \ci_{\m{M}} j | L$ and $I \ci_{\m{M}} k | L$ follows from (\ref{eq:2.17}) by definition of ${\m{M}}$.

\end{enumerate}

Therefore, we have proven the result when $|J|=|K|=1$. Assume as induction hypothesis that the result also holds when $2 < |J K| < s$. Assume without loss of generality that $1 < |J|$. Let $J = J_1 J_2$ such that $J_1, J_2 \neq \emptyset$ and $J_1 \cap J_2 = \emptyset$. 

\begin{enumerate}\setcounter{enumi}{4}

\item $I \ci_{\m{M}} J_1 | J_2 K L$ and $I \ci_{\m{M}} J_2 | J_1 K L$ by CI1 on $I \ci_{\m{M}} J | K L$.\label{eq:2.19}

\item $I \ci_{\m{M}} J_1 | J_2 L$ and $I \ci_{\m{M}} J_2 | J_1 L$ by the induction hypothesis on (\ref{eq:2.19}) and $I \ci_{\m{M}} K | J L$.\label{eq:2.20}

\item $I \ci_{\m{M}} J_1 | L$ by the induction hypothesis on (\ref{eq:2.20}).\label{eq:2.23}

\item $I \ci_{\m{M}} J | L$ by CI1 on (\ref{eq:2.20}) and (\ref{eq:2.23}).\label{eq:2.23b}

\item $I \ci_{\m{M}} K | L$ by CI1 on (\ref{eq:2.23b}) and $I \ci_{\m{M}} K | J L$.

\end{enumerate}

\underline{Proof of ci0-3 $\Rightarrow$ CI3}

\begin{enumerate}\setcounter{enumi}{9}

\item Assume that $I \ci_{\m{M}} j | L$ and $I \ci_{\m{M}} k | L$.\label{eq:2.28}

\item $i \ci_{\m{E}} j | M$ and $i \ci_{\m{E}} k | M$ for all $i \in I$ and $L \subseteq M \subseteq (I \setminus i) L$ follows from (\ref{eq:2.28}) by definition of ${\m{M}}$.\label{eq:2.30}

\item $i \ci_{\m{E}} j | k M$ and $i \ci_{\m{E}} k | j M$ for all $i \in I$ and $L \subseteq M \subseteq (I \setminus i) L$ by ci3 on (\ref{eq:2.30}).\label{eq:2.31}

\item $I \ci_{\m{M}} j | k L$ and $I \ci_{\m{M}} k |j L$ follows from (\ref{eq:2.31}) by definition of ${\m{M}}$.

\end{enumerate}

Therefore, we have proven the result when $|J|=|K|=1$. Assume as induction hypothesis that the result also holds when $2 < |J K| < s$. Assume without loss of generality that $1 < |J|$. Let $J = J_1 J_2$ such that $J_1, J_2 \neq \emptyset$ and $J_1 \cap J_2 = \emptyset$. 

\begin{enumerate}\setcounter{enumi}{13}

\item $I \ci_{\m{M}} J_1 | L$ by CI1 on $I \ci_{\m{M}} J | L$.\label{eq:2.33}

\item $I \ci_{\m{M}} J_2 | J_ 1 L$ by CI1 on $I \ci_{\m{M}} J | L$.\label{eq:2.34}

\item $I \ci_{\m{M}} K | J_1 L$ by the induction hypothesis on (\ref{eq:2.33}) and $I \ci_{\m{M}} K | L$.\label{eq:2.35}

\item $I \ci_{\m{M}} K | J L$ by the induction hypothesis on (\ref{eq:2.34}) and (\ref{eq:2.35}).\label{eq:2.36}

\item $I \ci_{\m{M}} J K | L$ by CI1 on (\ref{eq:2.36}) and $I \ci_{\m{M}} J | L$.\label{eq:2.37}

\item $I \ci_{\m{M}} J | K L$ and $I \ci_{\m{M}} K | J L$ by CI1 on (\ref{eq:2.37}).

\end{enumerate}
\end{proof}

The following lemma generalizes Lemma \ref{lem:G2P} by removing the assumptions about $\m{M}$ and $\m{E}$.

\begin{lemma}\label{lem:G2Pb}
Let $\m{M}$ denote a set of triplets. Then, $\om{E}^* = e(\m{M}^*)$, $\m{M}^* = m(\om{E}^*)$ and $\om{E}^* = \{ i \ci j | K : i \ci_{\m{M}^*} j | K \}$. Let $\m{E}$ denote a set of elementary triplets. Then, $\om{M}^* = m(\m{E}^*)$, $\m{E}^* = e(\om{M}^*)$ and $\m{E}^* = \{ i \ci j | K : i \ci_{\om{M}^*} j | K \}$.
\end{lemma}
 
\begin{proof}

Clearly, $\m{M} \subseteq m(\om{E}^*)$ and, thus, $\m{M}^* \subseteq m(\om{E}^*)$ because $m(\om{E}^*)$ satisfies CI0-1/CI0-2/CI0-3 by Lemma \ref{lem:G2P}. Clearly, $\om{E} \subseteq e(\m{M}^*)$ and, thus, $\om{E}^* \subseteq e(\m{M}^*)$ because $e(\m{M}^*)$ satisfies ci0-1/ci0-2/ci0-3 by Lemma \ref{lem:G2P}. Then, $\m{M}^* \subseteq m(\om{E}^*) \subseteq m(e(\m{M}^*))$ and $\om{E}^* \subseteq e(\m{M}^*) \subseteq e(m(\om{E}^*))$. Then, $\m{M}^* = m(\om{E}^*)$ and $\om{E}^* = e(\m{M}^*)$, because $\m{M}^* = m(e(\m{M}^*))$ and $\om{E}^* = e(m(\om{E}^*))$ by Lemma \ref{lem:G2P}. Finally, that $\om{E}^* = \{ i \ci j | K : i \ci_{\m{M}^*} j | K \}$ is now trivial.

Similarly, $\m{E} \subseteq e(\om{M}^*)$ and, thus, $\m{E}^* \subseteq e(\om{M}^*)$ because $e(\om{M}^*)$ satisfies ci0-1/ci0-2/ci0-3 by Lemma \ref{lem:G2P}. Clearly, $\om{M} \subseteq m(\m{E}^*)$ and, thus, $\om{M}^* \subseteq m(\m{E}^*)$ because $m(\m{E}^*)$ satisfies CI0-1/CI0-2/CI0-3 by Lemma \ref{lem:G2P}. Then, $\m{E}^* \subseteq e(\om{M}^*) \subseteq e(m(\m{E}^*))$ and $\om{M}^* \subseteq m(\m{E}^*) \subseteq m(e(\om{M}^*))$. Then, $\m{E}^* = e(\om{M}^*)$ and $\om{M}^* = m(\m{E}^*)$, because $\m{E}^* = e(m(\m{E}^*))$ and $\om{M}^* = m(e(\om{M}^*))$ by Lemma \ref{lem:G2P}. Finally, that $\m{E}^* = \{ i \ci j | K : i \ci_{\om{M}^*} j | K \}$ is now trivial.
\end{proof}

Lemma \ref{lem:G2P} implies that every set of triplets $\m{M}$ satisfying CI0-1/CI0-2/CI0-3 can be paired to a set of elementary triplets $\om{E}$ satisfying ci0-1/ci0-2/ci0-3, and vice versa. The lemma implies that the pairing is actually a bijection. Thanks to this bijection, we can use $\om{E}$ to represent $\m{M}$. This is in general a much more economical representation: If $|V| = n$, then there are up to $4^n$ triplets,\footnote{A triplet can be represented as a $n$-tuple whose entries state if the corresponding element is in the first, second, third or none set of the triplet.} whereas there are $n^2 \cdot 2^{n-2}$ elementary triplets at most.

Likewise, Lemma \ref{lem:G2Pb} implies that there is a bijection between the CI0-1/CI0-2/CI0-3 closures of sets of triplets and the ci0-1/ci0-2/ci0-3 closures of sets of elementary triplets. Thanks to this bijection, we can use $\om{E}^*$ to represent $\m{M}^*$. Note that $\om{E}^*$ is obtained by ci0-1/ci0-2/ci0-3 closing $\om{E}$, which is obtained from $\m{M}$. So, there is no need to CI0-1/CI0-2/CI0-3 close $\m{M}$ and so produce $\m{M}^*$. Whether closing $\om{E}$ can be done faster than closing $\m{M}$ on average is an open question. In the worst-case scenario, both imply applying the corresponding properties a number of times exponential in $|V|$ \cite{Matus02}. The following examples illustrate the savings in space that can be achieved by using $\om{E}^*$ to represent $\m{M}^*$.

\begin{example}\label{ex:one}
This example is taken from \cite{LopatatzidisvanderGaag}. Let $V=\{1,2,3,4,5,6\}$. Let $\m{M}=\{ 5 \ci 6 | \emptyset, 12 \ci 34 | 6, 23 \ci 14 | 5, 12 \ci 34 | 5, 3 \ci 14 | 25 \}$. The CI0-1, CI0-2 and CI0-3 closures of $\m{M}$ have the same 162 triplets. However, they can be represented in a more concise manner by their 82 elementary triplets.
\end{example}

\begin{example}\label{ex:two}
This example will be used again later in this work. Let $V=\{1,2,3,4,5,6\}$. Let $\m{M}=\{ 12 \ci456 | \emptyset, 123 \ci 4 | \emptyset \}$. The CI0-1, CI0-2 and CI0-3 closures of $\m{M}$ have the same 218 triplets. However, they can be represented in a more concise manner by their 112 elementary triplets.
\end{example}

One may think that Lemmas \ref{lem:G2P} and \ref{lem:G2Pb} have theoretical interest but little practical interest, because one may have access to a set of triplets $\m{M}$ that is not closed under CI0-1/CI0-2/CI0-3 and, thus, $\om{E}^*$ can only be obtained by first producing the CI0-1/CI0-2/CI0-3 closure of $\m{M}$ or as the ci0-1/ci0-2/ci0-3 closure of $\om{E}$. As mentioned above, the worst-case scenario for either alternative is computationally demanding. The complexity of the average case is unknown. However, we believe that Lemmas \ref{lem:G2P} and \ref{lem:G2Pb} are of practical interest when all one has access to is a probability distribution $p(V)$, e.g. the empirical distribution derived from a sample. In that case, the independence model induced by $p(V)$ can be represented by the elementary triplets $i \ci j | K$ such that $i \ci_p j | K$ holds. To see it, recall from Section \ref{sec:preliminaries} that the independence model induced by a probability distribution always satisfies the CI0-1 properties. Note that the process of finding the elementary triplets may be sped up by using the ci0-1 properties to derive elementary triplets from previously obtained elementary triplets, and so avoiding checking some pairwise independences in $p(V)$. One can instead use the ci0-2 or ci0-3 properties if it is known that $p(V)$ is strictly positive or regular Gaussian. This speeding up is warranted from the fact that the elementary triplet representation must be closed under ci0-1/ci0-2/ci0-3 by Lemmas \ref{lem:G2P} and \ref{lem:G2Pb}. For instance, having found that $i \ci_p j | \emptyset$ holds implies that $i \ci j | \emptyset$ must in the representation of the independence model induced by $p(V)$, which implies that so does $j \ci i | \emptyset$ by ci0. So, there is no need to check whether $j \ci_p i | \emptyset$ holds. This approach (without the speeding up sketched) of representing the independence model induced by a probability distribution with its elementary triplets has been instrumental in developing exact and assumption-free learning algorithms for chain graphs and acyclic directed mixed graphs \cite{Penna2016,Sonntagetal.2015}. One may argue that there is no need to produce a concise representation of $p(V)$ such as the elementary triplet representation, since it takes time and storage space and it provides no additional information about $p(V)$. However, some operations with independence models are not easy to perform without representing the independence models explicitly, e.g. it is not clear to us how to compute the intersection of the independence models induced by two probability distributions without representing the independence models in any way whereas, as we will see in Section \ref{sec:operations}, this is a straightforward question to answer from their elementary triplet representations. 

For simplicity, all the results in the sequel assume that $\m{M}$ and $\m{E}$ satisfy CI0-1/CI0-2/CI0-3 and ci0-1/ci0-2/ci0-3. Thanks to Lemma \ref{lem:G2Pb}, these assumptions can be dropped by replacing $\m{M}$, $\m{E}$, $\om{M}$ and $\om{E}$ in the results below with $\m{M}^*$, $\m{E}^*$, $\om{M}^*$ and $\om{E}^*$.

Let $I = i_1 \ldots i_m$ and $J = j_1 \ldots j_n$. In order to decide whether $I \ci_{\om{M}} J | K$, the definition of ${\om{M}}$ implies checking whether $m \cdot n \cdot 2^{(m+n-2)}$ elementary triplets are in ${\m{E}}$. The following lemma simplifies this when ${\m{E}}$ satisfies ci0-1, as it implies checking $m \cdot n$ elementary triplets. When ${\m{E}}$ satisfies ci0-2 or ci0-3, the lemma simplifies the decision even further as the conditioning sets of the elementary triplets checked have all the same size or form. 

\begin{lemma}\label{lem:simplification}
Let $\m{E}$ denote a set of elementary triplets. Let $\om{M}_1 = \{ I \ci J | K : i_s \ci_{\m{E}} j_t | i_1 \ldots i_{s-1} j_1 \ldots j_{t-1} K$ for all $1 \leq s \leq m$ and $1 \leq t \leq n \}$, $\om{M}_2 = \{ I \ci J | K : i \ci_{\m{E}} j | (I \setminus i) (J \setminus j) K$ for all $i \in I$ and $j \in J \}$, and $\om{M}_3 = \{ I \ci J | K : i \ci_{\m{E}} j | K$ for all $i \in I$ and $j \in J \}$. If ${\m{E}}$ satisfies ci0-1, then ${\om{M}} = \om{M}_1$. If ${\m{E}}$ satisfies ci0-2, then ${\om{M}} = \om{M}_2$. If ${\m{E}}$ satisfies ci0-3, then ${\om{M}} = \om{M}_3$.
\end{lemma}

\begin{proof}

\underline{Proof for ci0-1}

It suffices to prove that $\om{M}_1 \subseteq {\om{M}}$ because clearly ${\om{M}} \subseteq \om{M}_1$. Assume that $I \ci_{\om{M}_1} J | K$. Then, $i_s \ci_{\m{E}} j_t | i_1 \ldots i_{s-1} j_1 \ldots j_{t-1} K$ and $i_s \ci_{\m{E}} j_{t+1} | i_1 \ldots i_{s-1} j_1 \ldots j_{t} K$ by definition of $\om{M}_1$. Then, $i_s \ci_{\m{E}} j_{t+1} | i_1 \ldots i_{s-1} j_1 \ldots j_{t-1} K$ and $i_s \ci_{\m{E}} j_t | i_1 \ldots i_{s-1}$ $j_1 \ldots j_{t-1} j_{t+1} K$ by ci1. Then, $i_s \ci_{\om{M}} j_{t+1} | i_1 \ldots i_{s-1} j_1 \ldots j_{t-1} K$ and $i_s \ci_{\om{M}} j_t | i_1 \ldots i_{s-1}$ $j_1 \ldots j_{t-1} j_{t+1} K$ by definition of ${\om{M}}$. By repeating this reasoning, we can then conclude that $i_s \ci_{\om{M}} j_{\sigma(t)} | i_1 \ldots i_{s-1} j_{\sigma(1)} \ldots j_{\sigma(t-1)} K$ for any permutation $\sigma$ of the set $\{ 1 \ldots n \}$. By following an analogous reasoning for $i_s$ instead of $j_t$, we can then conclude that $i_{\varsigma(s)} \ci_{\om{M}} j_{\sigma(t)} | i_{\varsigma(1)} \ldots i_{\varsigma(s-1)} j_{\sigma(1)} \ldots j_{\sigma(t-1)} K$ for any permutations $\sigma$ and $\varsigma$ of the sets $\{ 1 \ldots n \}$ and $\{ 1 \ldots m \}$. This implies the desired result by definition of ${\om{M}}$.

\underline{Proof for ci0-2}

It suffices to prove that $\om{M}_2 \subseteq {\om{M}}$ because clearly ${\om{M}} \subseteq \om{M}_2$. Note that ${\om{M}}$ satisfies CI0-2 by Lemma \ref{lem:G2P}. Assume that $I \ci_{\om{M}_2} J | K$.

\begin{enumerate}

\item $i_1 \ci_{\om{M}} j_1 | (I \setminus i_1) (J \setminus j_1) K$ and $i_1 \ci_{\om{M}} j_2 | (I \setminus i_1) (J \setminus j_2) K$ follow from $i_1 \ci_{\m{E}}$ $j_1 | (I \setminus i_1) (J \setminus j_1) K$ and $i_1 \ci_{\m{E}} j_2 | (I \setminus i_1) (J \setminus j_2) K$ by definition of ${\om{M}}$.\label{eq:4.1}

\item $i_1 \ci_{\om{M}} j_1 | (I \setminus i_1) (J \setminus j_1 j_2) K$ by CI2 on (\ref{eq:4.1}), which together with (\ref{eq:4.1}) imply $i_1 \ci_{\om{M}} j_1 j_2 | (I \setminus i_1) (J \setminus j_1 j_2) K$ by CI1.\label{eq:4.2}

\item $i_1 \ci_{\om{M}} j_3 | (I \setminus i_1) (J \setminus j_3) K$ follows from $i_1 \ci_{\m{E}} j_3 | (I \setminus i_1) (J \setminus j_3) K$ by definition of ${\om{M}}$.\label{eq:4.3}

\item $i_1 \ci_{\om{M}} j_1 j_2 | (I \setminus i_1) (J \setminus j_1 j_2 j_3) K$ by CI2 on (\ref{eq:4.2}) and (\ref{eq:4.3}), which together with (\ref{eq:4.3}) imply $i_1 \ci_{\om{M}} j_1 j_2 j_3 | (I \setminus i_1) (J \setminus j_1 j_2 j_3) K$ by CI1.

\end{enumerate}

By continuing with the reasoning above, we can conclude that $i_1 \ci_{\om{M}} J | (I \setminus i_1) K$. Moreover, $i_2 \ci_{\om{M}} J | (I \setminus i_2) K$ by a reasoning similar to (1-4) and, thus, $i_1 i_2 \ci_{\om{M}} J | (I \setminus i_1 i_2) K$ by an argument similar to (2). Moreover, $i_3 \ci_{\om{M}} J | (I \setminus i_3) K$ by a reasoning similar to (1-4) and, thus, $i_1 i_2 i_3 \ci_{\om{M}} J | (I \setminus i_1 i_2 i_3) K$ by an argument similar to (4). Continuing with this process gives the desired result.

\underline{Proof for ci0-3}

It suffices to prove that $\om{M}_3 \subseteq {\om{M}}$ because clearly ${\om{M}} \subseteq \om{M}_3$. Note that ${\om{M}}$ satisfies CI0-3 by Lemma \ref{lem:G2P}. Assume that $I \ci_{\om{M}_3} J | K$.

\begin{enumerate}\setcounter{enumi}{4}

\item $i_1 \ci_{\om{M}} j_1 | K$ and $i_1 \ci_{\om{M}} j_2 | K$ follow from $i_1 \ci_{\m{E}} j_1 | K$ and $i_1 \ci_{\m{E}} j_2 | K$ by definition of ${\om{M}}$.\label{eq:5.1}

\item $i_1 \ci_{\om{M}} j_1 | j_2 K$ by CI3 on (\ref{eq:5.1}), which together with (\ref{eq:5.1}) imply $i_1 \ci_{\om{M}} j_1 j_2 | K$ by CI1.\label{eq:5.2}

\item $i_1 \ci_{\om{M}} j_3 | K$ follows from $i_1 \ci_{\m{E}} j_3 | K$ by definition of ${\om{M}}$.\label{eq:5.3}

\item $i_1 \ci_{\om{M}} j_1 j_2 | j_3 K$ by CI3 on (\ref{eq:5.2}) and (\ref{eq:5.3}), which together with (\ref{eq:5.3}) imply $i_1 \ci_{\om{M}}$ $j_1 j_2 j_3 | K$ by CI1.

\end{enumerate}

By continuing with the reasoning above, we can conclude that $i_1 \ci_{\om{M}} J | K$. Moreover, $i_2 \ci_{\om{M}} J | K$ by a reasoning similar to (5-8) and, thus, $i_1 i_2 \ci_{\om{M}} J | K$ by an argument similar to (6). Moreover, $i_3 \ci_{\om{M}} J | K$ by a reasoning similar to (5-8) and, thus, $i_1 i_2 i_3 \ci_{\om{M}} J | K$ by an argument similar to (8). Continuing with this process gives the desired result.
\end{proof}

As mentioned in the introduction, another set of distinguished triplets in $\m{M}$ that can be used to represent it is the set of dominant triplets \cite{BaiolettiBV09,deWaal:2004:SIC:1036843.1036857,LopatatzidisvanderGaag,studeny1998complexity}. The following lemma shows how to find these triplets with the help of ${\om{E}}$.

\begin{lemma}\label{lem:dominant}
Let $\m{M}$ denote a set of triplets. If $\m{M}$ satisfies CI0-1, then $I \ci J | K$ is a dominant triplet in $\m{M}$ if and only if $I = i_1 \ldots i_m$ and $J = j_1 \ldots j_n$ are two maximal sets such that $i_s \ci_{\om{E}} j_t | i_1 \ldots i_{s-1} j_1 \ldots j_{t-1} K$ for all $1 \leq s \leq m$ and $1 \leq t \leq n$ and, for all $k \in K$, $i_s \nci_{\om{E}} k | i_1 \ldots i_{s-1} J (K \setminus k)$ and $k \nci_{\om{E}} j_t | I j_1 \ldots j_{t-1} (K \setminus k)$ for some $1 \leq s \leq m$ and $1 \leq t \leq n$. If $\m{M}$ satisfies CI0-2, then $I \ci J | K$ is a dominant triplet in $\m{M}$ if and only if $I$ and $J$ are two maximal sets such that $i \ci_{\om{E}} j | (I \setminus i) (J \setminus j) K$ for all $i \in I$ and $j \in J$ and, for all $k \in K$, $i \nci_{\om{E}} k | (I \setminus i) J (K \setminus k)$ and $k \nci_{\om{E}} j | I (J \setminus j) (K \setminus k)$ for some $i \in I$ and $j \in J$. If $\m{M}$ satisfies CI0-3, then $I \ci J | K$ is a dominant triplet in $\m{M}$ if and only if $I$ and $J$ are two maximal sets such that $i \ci_{\om{E}} j | K$ for all $i \in I$ and $j \in J$ and, for all $k \in K$, $i \nci_{\om{E}} k | K \setminus k$ and $k \nci_{\om{E}} j | K \setminus k$ for some $i \in I$ and $j \in J$.
\end{lemma}

\begin{proof}

We prove the lemma when $\m{M}$ satisfies CI0-1. The other two cases can be proven in much the same way. To see the if part, note that $I \ci_{\m{M}} J | K$ by Lemmas \ref{lem:G2P} and \ref{lem:simplification}. Moreover, assume to the contrary that there is a triplet $I' \ci_{\m{M}} J' | K'$ that dominates $I \ci_{\m{M}} J | K$. Consider the following two cases: $K' = K$ and $K' \subset K$. In the first case, CI0-1 on $I' \ci_{\m{M}} J' | K'$ implies that $I i_{m+1} \ci_{\m{M}} J | K$ or $I \ci_{\m{M}} J j_{n+1} | K$ with $i_{m+1} \in I' \setminus I$ and $j_{n+1} \in J' \setminus J$. Assume the latter without loss of generality. Then, CI0-1 implies that $i_s \ci_{\om{E}} j_t | i_1 \ldots i_{s-1} j_1 \ldots j_{t-1} K$ for all $1 \leq s \leq m$ and $1 \leq t \leq n+1$. This contradicts the maximality of $J$. In the second case, CI0-1 on $I' \ci_{\m{M}} J' | K'$ implies that $I k \ci_{\m{M}} J | K \setminus k$ or $I \ci_{\m{M}} J k | K \setminus k$ with $k \in K$. Assume the latter without loss of generality. Then, CI0-1 implies that $i_s \ci_{\om{E}} k | i_1 \ldots i_{s-1} J (K \setminus k)$ for all $1 \leq s \leq m$, which contradicts the assumptions of the lemma.

To see the only if part, note that CI0-1 implies that $i_s \ci_{\om{E}} j_t | i_1 \ldots i_{s-1} j_1 \ldots j_{t-1} K$ for all $1 \leq s \leq m$ and $1 \leq t \leq n$. Moreover, assume to the contrary that for some $k \in K$, $i_s \ci_{\om{E}} k | i_1 \ldots i_{s-1} J (K \setminus k)$ for all $1 \leq s \leq m$ or $k \ci_{\om{E}} j_t | I j_1 \ldots j_{t-1} (K \setminus k)$ for all $1 \leq t \leq n$. Assume the latter without loss of generality. Then, $I k \ci_{\m{M}} J | K \setminus k$ by Lemmas \ref{lem:G2P} and \ref{lem:simplification}, which implies that $ I \ci_{\m{M}} J | K$ is not a dominant triplet in ${\m{M}}$, which is a contradiction. Finally, note that $I$ and $J$ must be maximal sets satisfying the properties proven in this paragraph because, otherwise, the previous paragraph implies that there is a triplet in ${\m{M}}$ that dominates $I \ci_{\m{M}} J | K$.
\end{proof}

A natural question to ponder is whether it is better to represent an independence model by its elementary or dominant triplets. In terms of storage space, it seems that the dominant triplet representation should be preferred. For instance, for the independence model in Example \ref{ex:one}, there are 82 elementary triplets but only 12 dominant triplets and nine non-symmetric dominant triplets \cite{LopatatzidisvanderGaag}. For the independence model in Example \ref{ex:two}, there are 112 elementary triplets but only two non-symmetric dominant triplets, as we will see later. In terms of running time, the answer is less clear. As mentioned before, finding $\om{E}^*$ for a given set of triplets $\m{M}$ implies producing the CI0-1/CI0-2/CI0-3 closure of $\m{M}$ or the ci0-1/ci0-2/ci0-3 closure of $\om{E}$. The average case complexity of either case is unknown. The algorithms in \cite{BaiolettiBV09,deWaal:2004:SIC:1036843.1036857,LopatatzidisvanderGaag,studeny1998complexity} for finding the dominant triplets in $\m{M}$ are conceptually more involved but they could be faster than finding $\om{E}^*$. Performing an empirical comparison of the two alternatives is definitely an interesting research project. However, it is beyond the scope of this work. Moreover, the methods for finding dominant triplets take a set of triplets as input. It is not clear to us how to run them when all we have access to is a probability distribution $p(V)$, e.g. the empirical distribution derived from a sample. As discussed before, finding the elementary triplet representation in that scenario is conceptually easy. Yet another dimension to compare elementary and dominant triplet representations is the operations that each alternative allows to perform efficiently, e.g. there is no method to our knowledge for computing the intersection of the CI0-1 closures of two sets of triplets when all we have is their dominant triplet representations whereas, as we will see in Section \ref{sec:operations}, this is a straightforward question to answer from their elementary triplet representations. That is why we prefer to see elementary and dominant triplets as complementary rather than competing alternatives to represent independence models: Depending on task at hand, one or the other may be preferred.

\begin{figure}[!t]
\scalebox{0.72}{ 
\begin{tikzpicture}[inner sep=1mm]
\node at (-2,5) (L) {$1 \ci_{\om{E}} 5 | 6$};
\node at (4,5) (L2) {$2 \ci_{\om{E}} 5 | 16$};
\node at (0,4) (K) {$1 \ci_{\om{E}} 6 | \emptyset$};
\node at (0,3) (J) {$1 \ci_{\om{E}} 4 | 6$};
\node at (2,4) (K2) {$2 \ci_{\om{E}} 6 | 1$};
\node at (2,3) (J2) {$2 \ci_{\om{E}} 4 | 16$};
\node at (0,2) (I) {$1 \ci_{\om{E}} 5 | 46$};
\node at (0,1) (H) {$1 \ci_{\om{E}} 6 | 4$};
\node at (2,2) (I2) {$2 \ci_{\om{E}} 5 | 146$};
\node at (2,1) (H2) {$2 \ci_{\om{E}} 6 | 14$};
\node at (0,0) (A) {$1 \ci_{\om{E}} 4 | \emptyset$};
\node at (0,-1) (B) {$1 \ci_{\om{E}} 5 | 4$};
\node at (0,-2) (C) {$1 \ci_{\om{E}} 6 | 45$};
\node at (0,-3) (D) {$1 \ci_{\om{E}} 4 | 5$};
\node at (0,-4) (E) {$1 \ci_{\om{E}} 5 | \emptyset$};
\node at (0,-5) (F) {$1 \ci_{\om{E}} 6 | 5$};
\node at (-2,-6) (G) {$1 \ci_{\om{E}} 4 | 56$};
\node at (2,0) (A2) {$2 \ci_{\om{E}} 4 | 1$};
\node at (2,-1) (B2) {$2 \ci_{\om{E}} 5 | 14$};
\node at (2,-2) (C2) {$2 \ci_{\om{E}} 6 | 145$};
\node at (2,-3) (D2) {$2 \ci_{\om{E}} 4 | 15$};
\node at (2,-4) (E2) {$2 \ci_{\om{E}} 5 | 1$};
\node at (2,-5) (F2) {$2 \ci_{\om{E}} 6 | 15$};
\node at (4,-6) (G2) {$2 \ci_{\om{E}} 4 | 156$};

\node at (6,-7) (A3) {$3 \ci_{\om{E}} 4 | 12$};
\node at (-4,0) (B3) {$3 \ci_{\om{E}} 4 | 1$};
\node at (-6,0) (C3) {$2 \ci_{\om{E}} 4 | 13$};
\node at (-4,-13) (B3b) {$3 \ci_{\om{E}} 4 | 2$};
\node at (-6,-13) (C3b) {$1 \ci_{\om{E}} 4 | 23$};
\node at (-6,-7) (D3) {$3 \ci_{\om{E}} 4 | \emptyset$};
\node at (-6,-6) (E3) {$1 \ci_{\om{E}} 4 | 3$};
\node at (-6,-8) (E3b) {$2 \ci_{\om{E}} 4 | 3$};

\node at (-2,-8) (Lb) {$2 \ci_{\om{E}} 5 | 6$};
\node at (4,-8) (L2b) {$1 \ci_{\om{E}} 5 | 26$};
\node at (0,-9) (Kb) {$2 \ci_{\om{E}} 6 | \emptyset$};
\node at (0,-10) (Jb) {$2 \ci_{\om{E}} 4 | 6$};
\node at (2,-9) (K2b) {$1 \ci_{\om{E}} 6 | 2$};
\node at (2,-10) (J2b) {$1 \ci_{\om{E}} 4 | 26$};
\node at (0,-11) (Ib) {$2 \ci_{\om{E}} 5 | 46$};
\node at (0,-12) (Hb) {$2 \ci_{\om{E}} 6 | 4$};
\node at (2,-11) (I2b) {$1 \ci_{\om{E}} 5 | 246$};
\node at (2,-12) (H2b) {$1 \ci_{\om{E}} 6 | 24$};
\node at (0,-13) (Ab) {$2 \ci_{\om{E}} 4 | \emptyset$};
\node at (0,-14) (Bb) {$2 \ci_{\om{E}} 5 | 4$};
\node at (0,-15) (Cb) {$2 \ci_{\om{E}} 6 | 45$};
\node at (0,-16) (Db) {$2 \ci_{\om{E}} 4 | 5$};
\node at (0,-17) (Eb) {$2 \ci_{\om{E}} 5 | \emptyset$};
\node at (0,-18) (Fb) {$2 \ci_{\om{E}} 6 | 5$};
\node at (-2,-19) (Gb) {$2 \ci_{\om{E}} 4 | 56$};
\node at (2,-13) (A2b) {$1 \ci_{\om{E}} 4 | 2$};
\node at (2,-14) (B2b) {$1 \ci_{\om{E}} 5 | 24$};
\node at (2,-15) (C2b) {$1 \ci_{\om{E}} 6 | 245$};
\node at (2,-16) (D2b) {$1 \ci_{\om{E}} 4 | 25$};
\node at (2,-17) (E2b) {$1 \ci_{\om{E}} 5 | 2$};
\node at (2,-18) (F2b) {$1 \ci_{\om{E}} 6 | 25$};
\node at (4,-19) (G2b) {$1 \ci_{\om{E}} 4 | 256$};

\path[->] (L) edge (G);
\path[->] (L2) edge (G2);
\path[->] (K) edge (L);
\path[->] (K2) edge (L2);
\path[->,dashed] (L) edge (L2);
\path[->] (K) edge (J);
\path[->] (J) edge (I);
\path[->] (K2) edge (J2);
\path[->] (J2) edge (I2);
\path[->,dashed] (K) edge (K2);
\path[->,dashed] (J) edge (J2);
\path[->] (A) edge (H);
\path[->] (H) edge (I);
\path[->] (A2) edge (H2);
\path[->] (H2) edge (I2);
\path[->,dashed] (H) edge (H2);
\path[->,dashed] (I) edge (I2);
\path[->] (A) edge (B);
\path[->] (B) edge (C);
\path[->] (D) edge (C);
\path[->] (E) edge (D);
\path[->] (E) edge (F);
\path[->] (F) edge (G);
\path[->] (A2) edge (B2);
\path[->] (B2) edge (C2);
\path[->] (D2) edge (C2);
\path[->] (E2) edge (D2);
\path[->] (E2) edge (F2);
\path[->] (F2) edge (G2);
\path[->,dashed] (A) edge (A2);
\path[->,dashed] (B) edge (B2);
\path[->,dashed] (C) edge (C2);
\path[->,dashed] (D) edge (D2);
\path[->,dashed] (E) edge (E2);
\path[->,dashed] (F) edge (F2);
\path[->,dashed] (G) edge (G2);

\path[->,dashed] (A2) edge [bend left] (A3);
\path[->,dashed] (A) edge (B3);
\path[->,dashed] (B3) edge (C3);
\path[->,dashed] (Ab) edge (B3b);
\path[->,dashed] (B3b) edge (C3b);
\path[->,dashed] (D3) edge (E3);
\path[->,dashed] (E3) edge (C3);
\path[->,dashed] (D3) edge (E3b);
\path[->,dashed] (E3b) edge (C3b);

\path[->] (Lb) edge (Gb);
\path[->] (L2b) edge (G2b);
\path[->] (Kb) edge (Lb);
\path[->] (K2b) edge (L2b);
\path[->,dashed] (Lb) edge (L2b);
\path[->] (Kb) edge (Jb);
\path[->] (Jb) edge (Ib);
\path[->] (K2b) edge (J2b);
\path[->] (J2b) edge (I2b);
\path[->,dashed] (Kb) edge (K2b);
\path[->,dashed] (Jb) edge (J2b);
\path[->] (Ab) edge (Hb);
\path[->] (Hb) edge (Ib);
\path[->] (A2b) edge (H2b);
\path[->] (H2b) edge (I2b);
\path[->,dashed] (Hb) edge (H2b);
\path[->,dashed] (Ib) edge (I2b);
\path[->] (Ab) edge (Bb);
\path[->] (Bb) edge (Cb);
\path[->] (Db) edge (Cb);
\path[->] (Eb) edge (Db);
\path[->] (Eb) edge (Fb);
\path[->] (Fb) edge (Gb);
\path[->] (A2b) edge (B2b);
\path[->] (B2b) edge (C2b);
\path[->] (D2b) edge (C2b);
\path[->] (E2b) edge (D2b);
\path[->] (E2b) edge (F2b);
\path[->] (F2b) edge (G2b);
\path[->,dashed] (Ab) edge (A2b);
\path[->,dashed] (Bb) edge (B2b);
\path[->,dashed] (Cb) edge (C2b);
\path[->,dashed] (Db) edge (D2b);
\path[->,dashed] (Eb) edge (E2b);
\path[->,dashed] (Fb) edge (F2b);
\path[->,dashed] (Gb) edge (G2b);

\path[->,dashed] (A2b) edge [bend right] (A3);
\end{tikzpicture}
}
\caption{Example of the DAG representation of ${\om{E}}$ (up to symmetry).}\label{fig:example}
\end{figure}
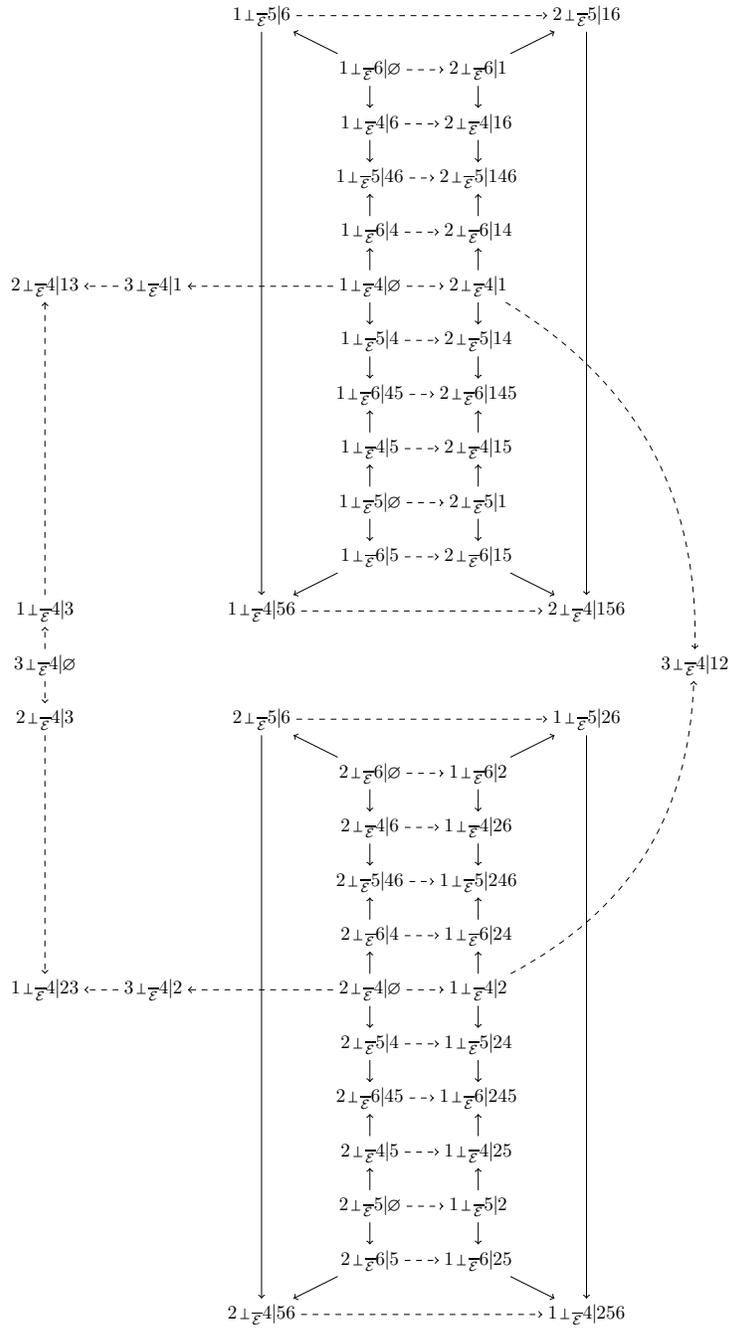

Inspired by \cite{Matus02}, if ${\m{M}}$ satisfies CI0-1 then we can represent ${\om{E}}$ as a DAG. The nodes of the DAG are the elementary triplets in ${\om{E}}$ and the edges of the DAG are $\{ i \ci_{\om{E}} k | L \ra i \ci_{\om{E}} j | k L \} \cup \{ k \ci_{\om{E}} j | L \dashedrightarrow i \ci_{\om{E}} j | k L \}$. See Figure \ref{fig:example} for an example. For the sake of readability, the DAG in the figure does not include symmetric elementary triplets. That is, the complete DAG can be obtained by adding a second copy of the DAG in the figure, replacing every node $i \ci_{\om{E}} j | K$ in the copy with $j \ci_{\om{E}} i | K$, and replacing every edge $\ra$ (respectively $\dashedrightarrow$) in the copy with $\dashedrightarrow$ (respectively $\ra$). We say that a subgraph over $m \cdot n$ nodes of the DAG is a grid if there is a bijection between the nodes of the subgraph and the labels $\{ v_{s,t} : 1 \leq s \leq m, 1 \leq t \leq n \}$ such that the edges of the subgraph are $\{ v_{s,t} \ra v_{s,t+1} :  1 \leq s \leq m, 1 \leq t < n \} \cup \{ v_{s,t} \dashedrightarrow v_{s+1,t} :  1 \leq s < m, 1 \leq t \leq n \}$. For instance, the following subgraph of the DAG in Figure \ref{fig:example} is a grid:

\begin{center}
\scalebox{0.72}{ 
\begin{tikzpicture}[inner sep=1mm]

\node at (0,0) (Bb) {$2 \ci_{\om{E}} 5 | 4$};
\node at (0,-1) (Cb) {$2 \ci_{\om{E}} 6 | 45$};

\node at (3,0) (B2b) {$1 \ci_{\om{E}} 5 | 24$};
\node at (3,-1) (C2b) {$1 \ci_{\om{E}} 6 | 245$};

\path[->] (Bb) edge (Cb);
\path[->] (B2b) edge (C2b);
\path[->,dashed] (Bb) edge (B2b);
\path[->,dashed] (Cb) edge (C2b);
\end{tikzpicture}
}
\end{center}

The following lemma is an immediate consequence of Lemmas \ref{lem:G2P} and \ref{lem:simplification}. 

\begin{lemma}\label{lem:grid}
Let ${\m{M}}$ denote a set of triplets that satisfies CI0-1, and let $I = i_1 \ldots i_m$ and $J = j_1 \ldots j_n$. If the subgraph of the DAG representation of ${\om{E}}$ induced by the set of nodes $\{ i_s \ci_{\om{E}}$ $j_t | i_1 \ldots i_{s-1} j_1 \ldots j_{t-1} K : 1 \leq s \leq m, 1 \leq t \leq n \}$ is a grid, then $I \ci_{\m{M}} J | K$.
\end{lemma}

Thanks to Lemmas \ref{lem:dominant} and \ref{lem:grid}, finding dominant triplets can now be reformulated as finding maximal grids in the DAG. Note that this is a purely graphical characterization. For instance, the DAG in Figure \ref{fig:example} has 18 maximal grids: The subgraphs induced by the set of nodes $\{ \sigma(s) \ci_{\om{E}} \varsigma(t) | \sigma(1) \ldots \sigma(s-1) \varsigma(1) \ldots \varsigma(t-1) : 1 \leq s \leq 2, 1 \leq t \leq 3 \}$ where $\sigma$ and $\varsigma$ are permutations of $\{ 1, 2 \}$ and $\{ 4, 5, 6 \}$, and the set of nodes $\{ \pi(s) \ci_{\om{E}} 4 | \pi(1) \ldots \pi(s-1) :  1 \leq s \leq 3 \}$ where $\pi$ is a permutation of $\{ 1, 2, 3 \}$. These grids correspond to the dominant triplets $12 \ci_{\m{M}} 456 | \emptyset$ and $123 \ci_{\m{M}} 4 | \emptyset$. It should be mentioned that the DAG representation of ${\om{E}}$ is a theoretical construct that, in its current form, brings little advantage in practice, since it can get quite large even for small domains.

\section{Operations}\label{sec:operations}

In this section, we discuss how some operations with independence models can be performed with the help of ${\om{E}}$.

\subsection{Membership}\label{sec:membership}

We want to check whether $I \ci_{\m{M}} J | K$, where ${\m{M}}$ denotes a set of triplets satisfying CI0-1/CI0-2/CI0-3. Recall that ${\m{M}}$ can be obtained from ${\om{E}}$ by Lemma \ref{lem:G2P}. Recall also that ${\om{E}}$ satisfies ci0-1/ci0-2/ci0-3 by Lemma \ref{lem:G2P} and, thus, Lemma \ref{lem:simplification} applies to ${\om{E}}$, which simplifies producing ${\m{M}}$ from ${\om{E}}$. Specifically if ${\m{M}}$ satisfies CI0-1, then we can check whether $I \ci_{\m{M}} J | K$ with $I = i_1 \ldots i_m$ and $J = j_1 \ldots j_n$ by checking whether $i_s \ci_{\om{E}} j_t | i_1 \ldots i_{s-1} j_1 \ldots j_{t-1} K$ for all $1 \leq s \leq m$ and $1 \leq t \leq n$. Thanks to Lemma \ref{lem:grid}, this solution can also be reformulated as checking whether the DAG representation of ${\om{E}}$ contains a suitable grid. Likewise, if ${\m{M}}$ satisfies CI0-2, then we can check whether $I \ci_{\m{M}} J | K$ by checking whether $i \ci_{\om{E}} j | (I \setminus i) (J \setminus j) K$ for all $i \in I$ and $j \in J$. Finally, if ${\m{M}}$ satisfies CI0-3, then we can check whether $I \ci_{\m{M}} J | K$ by checking whether $i \ci_{\om{E}} j | K$ for all $i \in I$ and $j \in J$. Note that in the last two cases, we only need to check elementary triplets with conditioning sets of a specific length or form.

\subsection{Minimal Independence Map}

Given a set of triplets $\m{M}$ that satisfies CI0-1, a minimal independence map (MIM) of $\m{M}$ relative to an ordering $\sigma$ of the elements in $V$ can be built by setting $Pa_G(\sigma(s))$ for all $1 \leq s \leq |V|$ to a minimal subset of $\sigma(1) \ldots \sigma(s-1)$ such that $\sigma(s) \ci_{\m{M}} \sigma(1) \ldots \sigma(s-1) \setminus Pa_G(\sigma(s)) | Pa_G(\sigma(s))$ \cite[Theorem 9]{Pearl:1988:PRI:534975}. A MIM can be built with the help of the DAG representation of ${\om{E}}$ as follows. First, let us define the function $AllPa(i,X)$ with $i \in V$ and $X \subseteq V \setminus i$ as follows. The function returns all the sets $Y \subseteq X$ that qualify as parents of $i$, i.e. $i \ci_{\m{M}} X \setminus Y | Y$.

\begin{table}[H]
\begin{tabular}{rl}
& $\underline{AllPa(i, X)}$\\
\\
1 & $aux= \emptyset$\\
2 & for each longest grid in the DAG representation of ${\om{E}}$ that is of the form\\
& $i \ci_{\om{E}} j_1 | X \setminus j_1 \ldots j_n \ra i \ci_{\om{E}} j_2 | X \setminus j_2 \ldots j_n \ra \ldots \ra i \ci_{\om{E}} j_n | X \setminus j_n$ or\\
& $j_1 \ci_{\om{E}} i | X \setminus j_1 \ldots j_n \dashedrightarrow j_2 \ci_{\om{E}} i | X \setminus j_2 \ldots j_n \dashedrightarrow \ldots \dashedrightarrow j_n \ci_{\om{E}} i | X \setminus j_n$ with\\
& $j_1 \ldots j_n \subseteq X$ do\\
3 & \hspace{0.3cm} $aux = aux \cup \{ X \setminus j_1 \ldots j_n \}$\\
4 & if $aux \neq \emptyset$ then\\
5 & \hspace{0.3cm} return $aux$\\
6 & else\\
7 & \hspace{0.3cm} return $X$\\
\end{tabular}
\end{table}

Note that for every set of nodes $Y \in AllPa(i,X)$, we have that $i \ci_{\m{M}} X \setminus Y | Y$ by Lemma \ref{lem:grid}. Therefore, building a MIM of ${\m{M}}$ relative to $\sigma$ can now be reformulated as setting $Pa_G(\sigma(s))=Y$ with $Y \in AllPa(\sigma(s),\sigma(1) \ldots \sigma(s-1))$ for all $1 \leq s \leq |V|$.

Since ${\m{M}}$ satisfies CI0-1, we can check whether the MIM built above is a perfect map (PM) of ${\m{M}}$ by checking whether ${\m{M}}$ coincides with the CI0-1 closure of $\{ \sigma(s) \ci \sigma(1) \ldots \sigma(s-1) \setminus Pa_G(\sigma(s)) | Pa_G(\sigma(s)) : 1 \leq s \leq |V| \}$ \cite[Corollary 7]{Pearl:1988:PRI:534975}. This result suggests the following method for checking whether ${\m{M}}$ has a PM: ${\m{M}}$ has a PM if and only if the call $PM(\emptyset, \emptyset)$ to the following function returns true.

\begin{table}[H]
\begin{tabular}{rl}
& $\underline{PM(Visited, Marked)}$\\
\\
1 & if $Visited = V$ then\\
2 & \hspace{0.3cm} if ${\om{E}}$ coincides with the ci0-1 closure of $Marked$\\
3 & \hspace{0.3cm} then return true and stop\\
4 & else\\
5 & \hspace{0.3cm} for each node $i \in V \setminus Visited$ do\\
6 & \hspace{0.6cm} for each $Pa \in AllPa(i,Visited)$ do\\
7 & \hspace{0.9cm} $PM(Visited \cup \{ i \}, Marked \cup e(\{ i \ci_{\m{M}} Visited \setminus Pa | Pa,$\\
& \hspace{6cm} $Visited \setminus Pa \ci_{\m{M}} i | Pa \}))$\\
\end{tabular}
\end{table}

Note that the function above is recursive. Lines 2-3 conform the trivial case, whereas lines 5-7 conform the recursive case. Lines 5 makes the function consider every ordering of the nodes in $V$ before stopping. For a particular ordering, line 6 considers all the subsets $Pa$ of the predecessors of the node $i$ in the ordering (i.e. $Visited$) that qualify as parents of $i$, i.e. $i \ci_{\m{M}} Visited \setminus Pa | Pa$. Such subsets are exactly the output of the function $AllPa(i,Visited)$. For each such subset $Pa$, line 7 marks $i$ as visited (i.e. processed), marks the elementary triplets used in the derivation of $i \ci_{\m{M}} Visited \setminus Pa | Pa$ and, then, it launches the search for the parents of the next node in the ordering by recursively calling the function. Note that the parameters are passed by value in the recursive call. Finally, note the need to compute the ci0-1 closure of $Marked$ in line 2. The elementary triplets in $Marked$ represent the triplets corresponding to the grids identified by the calls to the function $AllPa$ in line 6. However, it is the ci0-1 closure of the elementary triplets in $Marked$ that represents the CI0-1 closure of the triplets corresponding to the grids identified by the calls to the function $AllPa$, by Lemma \ref{lem:G2Pb}.

Finally, it is worth mentioning that if ${\m{M}}$ satisfies CI0-2, then there exist methods to build a MIM and check the existence of a PM that make use of the dominant triplets of ${\m{M}}$ \cite{Baioletti20112}.

\subsection{Inclusion}\label{sec:inclusion}

Let ${\m{M}}$ and ${\m{M}}'$ denote two sets of triplets satisfying CI0-1/CI0-2/CI0-3. We can check whether ${\m{M}} \subseteq {\m{M}}'$ by checking whether ${\om{E}} \subseteq {\om{E}}'$. In the view of Lemma \ref{lem:G2P}, this result is an immediate consequence of Lemma 2.2 by \cite{Studeny2005}. If the DAG representations of ${\om{E}}$ and ${\om{E}}'$ are available, then we can answer the inclusion question by checking whether the former is a subgraph of the latter.

\subsection{Intersection}

Let ${\m{M}}$ and ${\m{M}}'$ denote two sets of triplets satisfying CI0-1/CI0-2/CI0-3. Note that ${\m{M}} \cap {\m{M}}'$ satisfies CI0-1/CI0-2/CI0-3. Likewise, ${\om{E}} \cap {\om{E}}'$ satisfies ci0-1/ci0-2/ci0-3. We can represent ${\m{M}} \cap {\m{M}}'$ by ${\om{E}} \cap {\om{E}}'$. To see it, note that $I \ci_{{\m{M}} \cap {\m{M}}'} J | K$ if and only if $i \ci_{\om{E}} j | M$ and $i \ci_{{\om{E}}'} j | M$ for all $i \in I$, $j \in J$, and $K \subseteq M \subseteq (I \setminus i) (J \setminus j) K$. If the DAG representations of ${\om{E}}$ and ${\om{E}}'$ are available, then we can represent ${\m{M}} \cap {\m{M}}'$ by the subgraph of either of them induced by the nodes that are in both of them.

Typically, a single expert (or learning algorithm) is consulted to provide an independence model of the domain at hand. Hence the risk that the independence model may not be accurate, e.g. if the expert has some bias or overlooks some details. One way to minimize this risk consists in obtaining multiple independence models of the domain from multiple experts and, then, combining them into a single consensus independence model. In particular, we define the consensus independence model as the model that contains all and only the conditional independences on which all the given models agree, i.e. the intersection of the given models. Therefore, the paragraph above provides us with an efficient way to obtain the consensus independence model. When the given models are represented by their dominant triplets, an operator to obtain the consensus independence model exists for the case where the given models satisfy CI0-2 \cite{BaiolettiPV13}. The problem is harder if we only consider independence models induced by DAGs: There may be several non-equivalent consensus models, and finding one of them is NP-hard \cite[Theorems 1 and 2]{Penna2011}. So, one has to resort to heuristics.

\subsection{Context-specific Independences}\label{sec:csi}

Note that in a context-specific independence the context always appears in the conditioning set of the triplet. Thus, the results presented so far in this paper hold for independence models containing context-specific independences. We just need to rephrase the properties CI0-3 and ci0-3 to accommodate context-specific independences. We elaborate more on this in Section \ref{sec:csirevisited}.

\subsection{Union}

Let ${\m{M}}$ and ${\m{M}}'$ denote two sets of triplets satisfying CI0-1/CI0-2/CI0-3. Note that ${\m{M}} \cup {\m{M}}'$ may not satisfy CI0-1/CI0-2/CI0-3. For instance, let ${\m{M}} = \{ x \ci y | z, y \ci x | z \}$ and ${\m{M}}' = \{x \ci z | \emptyset, z \ci x | \emptyset \}$. Then, $x \ci y | z$ and $x \ci z | \emptyset$ are in ${\m{M}} \cup {\m{M}}'$ but $x \ci y z | \emptyset$ is not. A naive solution to this problem is simply introducing an auxiliary random variable $aux$ with domain $\{0,1\}$, and adding the context $aux=0$ (respectively $aux=1$) to the conditioning set of every triplet in ${\m{M}}$ (respectively ${\m{M}}'$). In the previous example, ${\m{M}} = \{ x \ci y | z, aux=0, y \ci x | z, aux=0 \}$ and ${\m{M}}' = \{x \ci z | aux=1, z \ci x | aux=1 \}$. Now, we can represent ${\m{M}} \cup {\m{M}}'$ by first adding the context $aux=0$ (respectively $aux=1$) to the conditioning set of every elementary triplet in ${\om{E}}$ (respectively ${\om{E}}'$) and, then, taking ${\om{E}} \cup {\om{E}}'$. This solution has advantages and disadvantages. The main advantage is that we represent ${\m{M}} \cup {\m{M}}'$ exactly. One of the disadvantages is that the same elementary triplet may appear twice in the representation, i.e. with different contexts in the conditioning set. Another disadvantage is that we need to modify slightly the procedures described above for building MIMs, and checking membership and inclusion. We believe that the advantage outweighs the disadvantages.

If the solution above is not satisfactory or it is deemed to lack a deeper justification, then we have two options: Representing a minimal superset or a maximal subset of ${\m{M}} \cup {\m{M}}'$ satisfying CI0-1/CI0-2/CI0-3. Note that the minimal superset of ${\m{M}} \cup {\m{M}}'$ satisfying CI0-1/CI0-2/CI0-3 is unique because, otherwise, the intersection of any two such supersets is a superset of ${\m{M}} \cup {\m{M}}'$ that satisfies CI0-1/CI0-2/CI0-3, which contradicts the minimality of the original supersets. On the other hand, the maximal subset of ${\m{M}} \cup {\m{M}}'$ satisfying CI0-1/CI0-2/CI0-3 is not unique. For instance, let ${\m{M}} = \{ x \ci y | z, y \ci x | z \}$ and ${\m{M}}' = \{ x \ci z | \emptyset, z \ci x | \emptyset \}$. Then, ${\m{M}} \cup {\m{M}}'$ does not satisfy CI1, e.g. $x \ci y | z$ and $x \ci z | \emptyset$ are in ${\m{M}} \cup {\m{M}}'$ but $x \ci y z | \emptyset$ is not. Moreover, both ${\m{M}}$ and ${\m{M}}'$ are maximal subsets of ${\m{M}} \cup {\m{M}}'$ that satisfy CI0-1/CI0-2/CI0-3, i.e. ${\m{M}}$ and ${\m{M}}'$ satisfy CI0-1/CI0-2/CI0-3 but, as shown, adding any triplet in ${\m{M}}'$ to ${\m{M}}$ or vice versa results in a set of triplets that does not satisfy CI1.

Coming back to the two options mentioned above, we can represent the minimal superset of ${\m{M}} \cup {\m{M}}'$ satisfying CI0-1/CI0-2/CI0-3 by the ci0-1/ci0-2/ci0-3 closure of ${\om{E}} \cup {\om{E}}'$. Clearly, this representation represents a superset of ${\m{M}} \cup {\m{M}}'$. Moreover, the superset satisfies CI0-1/CI0-2/CI0-3 by Lemma \ref{lem:G2P}. Minimality follows from the fact that removing any elementary triplet from the closure of ${\om{E}} \cup {\om{E}}'$ so that the result is still closed under ci0-1/ci0-2/ci0-3 implies removing some elementary triplet in ${\om{E}} \cup {\om{E}}'$, which implies not representing some triplet in ${\m{M}} \cup {\m{M}}'$ by Lemma \ref{lem:G2P}. Note that the DAG representation of ${\m{M}} \cup {\m{M}}'$ is not the union of the DAG representations of ${\om{E}}$ and ${\om{E}}'$, because we first have to close ${\om{E}} \cup {\om{E}}'$ under ci0-1/ci0-2/ci0-3. We can represent a maximal subset of ${\m{M}} \cup {\m{M}}'$ satisfying CI0-1/CI0-2/CI0-3 by a maximal subset $U$ of ${\om{E}} \cup {\om{E}}'$ that is closed under ci0-1/ci0-2/ci0-3 and such that every triplet represented by $U$ is in ${\m{M}} \cup {\m{M}}'$. Recall that we can check the latter as shown in Section \ref{sec:membership}. In fact, we do not need to check it for every triplet but only for the dominant triplets. Recall that these can be obtained from $U$ as shown in Lemma \ref{lem:dominant}. It should be noted that both options discussed in this paragraph can be computationally demanding since, as mentioned before, closing a set of elementary triplets under ci0-1/ci0-2/ci0-3 is demanding in the worst case scenario and the complexity of the average case is unknown.

Finally, it is worth mentioning that if ${\m{M}}$ and ${\m{M}}'$ satisfy CI0-2, then there exist methods to obtain from the dominant triplets of ${\m{M}}$ and ${\m{M}}'$ both the minimal superset and a maximal superset of ${\m{M}} \cup {\m{M}}'$ satisfying CI0-2 \cite{BaiolettiPV13}.

\subsection{Causal Reasoning}\label{sec:reasoning}

Causal reasoning comprises the study of cause and effect relationships, and the conditions under which they can be elucidated from observed quantities. For instance, we may be interested in quantifying the effect on a patient's health ($H$) of a prescribed treatment ($T=t$). In general, this effect does not coincide with the conditional probability distribution $p(H|t)$: The former accounts for the causal paths from $T$ to $H$, whereas the latter accounts for all the paths, which may include non-causal ones, e.g. if $T$ and $H$ have a common cause, say the socioeconomic status of the patient. The causal effect is typically denoted as $p(H|do(t))$ to indicate that $t$ is not an observation but an intervention, i.e. $T$ has been set to value $t$ independently of its causes and, thus, the non-causal paths from $T$ to $H$ should be ignored. As already seen in this toy example, it is rather natural to think of the causal relationships under study as directed edges in a graph. The graph may also contain bidirected edges to represent correlations due to unobserved common causes, also called confounders.

Since predicting the consequences of decisions or actions is necessary in many disciplines, it is not surprising that research on causal reasoning has a long tradition. Specifically, causal reasoning can be traced back to the work by Wright \cite{Wright1921}, where path analysis was introduced for the first time. Path analysis relies on the just described graphical representation of the causal model at hand. Moreover, the common effect of a set of causes is assumed to be a linear combination of the causes. Wright showed how to use the graph to perform causal reasoning. Apparently, a large part of the research community did not see much merit in Wright's graphical approach and preferred to work with the underlying system of linear equations. Later, the linearity constraint was lifted giving rise to a non-parametric structural equation model \cite{Haavelmo1943}. Another non-graphical approach to causal reasoning was developed by Neyman and Rubin, the so-called potential outcome model \cite{Neyman1923,Rubin1974}. This model has been shown to be subsumed by the non-parametric structural equation model \cite[Section 7.4.4]{Pearl2009}.

Wright's work was rediscovered in the 1980s by Pearl and other researchers. Their advances in the field are best reported in \cite{Pearl2009}. Although Pearl's work builds on path analysis, it differs from it in two significant aspects. First, the linearity assumption is dropped so that the causal models considered are non-parametric. Second, Pearl and co-workers succeeded in giving a sound and complete characterization of the conditions for a causal effect to be identifiable, i.e. computable from observed quantities. The characterization is graphical, meaning that it is expressed in terms of the graphical representation of the causal model of the domain under study.

As mentioned, the graphical approach to causal reasoning has produced very satisfactory results. However, it has two main disadvantages. First, it does not apply to domains whose causal model cannot be represented by a graph. Many domains arguably fall in this category. For instance, those domains that contain correlations that cannot be attributed to confounding, e.g. correlations due to selection bias, physical laws devoid of causal meaning, or feedback loops. Second, the graphical approach to causal reasoning makes an implicit modularity or invariance assumption, which does not always hold: The causes of a random variable do not change when we intervene on another random variable. See \cite{Dawid2010a,Dawid2010b} for further details on these problems. See also \cite{Dawid2015} for the outline of a decision theoretic approach to causal reasoning that overcomes the problems just described. Note that even Pearl acknowledges the need to develop non-graphical approaches to causal reasoning \cite[p. 10]{GallesandPearl1997}. 

Despite the many interesting results reported over the years for the non-graphical approaches to causal reasoning mentioned above, they lag behind the graphical approach in terms of meaningfulness and insightfulness. Inspired by the decision theoretic approach to causality in \cite{Dawid2015}, we present in this section our contribution to solve this problem. Specifically, we present a series of sufficient conditions for causal effect identification from the independence model of the domain at hand. We propose to represent the independence model by its elementary triplets, and so take advantage of the results reported in the previous sections of this work.

As in \cite[Section 3.2.2]{Pearl2009}, we start by adding an exogenous random variable $F_j$ for each $j \in V$, such that $F_j$ takes values in $\{interventional, observational\}$. These values represent whether an intervention has been performed on $j$ or not. We use $I_j$ and $O_j$ to denote respectively that $F_j = interventional$ and $F_j = observational$. The random variables in $F_V V$ are governed by a probability distribution $p(F_V V)$. We assume to have access to $p(V | O_V)$ only, e.g. through a sample of the observational regime. We aim to identify conditions that allow computing an expression of the form $p(Y | I_X O_{V \setminus X} X W)$ from $p(V | O_V)$, with $X$, $Y$ and $W$ disjoint subsets of $V$. These conditions will be expressed in terms of independences over subsets of $F_V V$. For instance, if $Y \ci_p F_X | O_{V \setminus X} X W$ then $p(Y | I_X O_{V \setminus X} X W) = p(Y | O_V X W)$. To check whether the conditions hold, we need therefore to have access to the independence model ${\m{M}}$ induced by $p(F_V V)$. However, recall that we only have access to $p(V | O_V)$. Therefore, we assume that the user will be able to provide us with ${\m{M}}$. We believe that the most convenient (albeit tedious) way of doing so is by providing us with $\om{E}$. As shown before, $\om{E}$ identifies ${\m{M}}$ unambiguously and is considerably more concise (Lemmas \ref{lem:G2P} and \ref{lem:G2Pb}), and it allows checking relatively efficiently whether an independence is in ${\m{M}}$ (Section \ref{sec:membership}). Moreover, it only requires specifying pairwise independences, which simplifies the task of the user. Of course, the user will make use of $p(V | O_V)$ to decide on those independences of the form $i \ci j | O_V Z$ with $i, j \in V$ and $Z \subseteq V \setminus i j$.

It should be mentioned that most of the conditional independences in this section will be context-specific, as they will include $O_V$ or $F_V$ in the conditioning set. Moreover, we assume that $p(V | O_V)$ is strictly positive. This prevents an intervention from setting a random variable to a value with zero probability under the observational regime, which would make our quest impossible. For the sake of readability, we assume that the random variables in $V$ are in their observational regimes unless otherwise stated. Thus, hereinafter $\tp(Y | I_X X W)$ is a shortcut for $p(Y | I_X O_{V \setminus X} X W)$, $Y \tci_{\m{M}} F_X | X W$ is a shortcut for $Y \ci_{\m{M}} F_X | O_{V \setminus X} X W$, and so on. The rest of this section shows how to perform causal reasoning with independence models by rephrasing some of the main results in \cite[Chapter 4]{Pearl2009} in terms of conditional independences alone, i.e. no causal graphs are involved.

\subsubsection{{\em do}-Calculus, and Back-Door and Front-Door Criteria}

We start by rephrasing Pearl's {\em do}-calculus \cite[Theorem 3.4.1]{Pearl2009}.

\begin{theorem}
Let $X$, $Y$, $W$ and $Z$ denote four disjoint subsets of $V$. Then
\begin{itemize}
\item Rule 1 (insertion/deletion of observations).

If $Y \tci_{\m{M}} X | I_Z W Z$ then $\tp(Y| I_Z X W Z) = \tp(Y| I_Z W Z)$.

\item Rule 2 (intervention/observation exchange). 

If $Y \tci_{\m{M}} F_X | I_Z X W Z$ then $\tp(Y| I_X I_Z X W Z) = \tp(Y| I_Z X W Z)$.

\item Rule 3 (insertion/deletion of interventions). 

If $Y \tci_{\m{M}} X | I_X I_Z W Z$ and $Y \tci_{\m{M}} F_X | I_Z W Z$, then $\tp(Y| I_X I_Z X W Z) = \tp(Y| I_Z W Z)$.
\end{itemize}
\end{theorem}

\begin{proof}
Rules 1 and 2 are immediate. To prove rule 3, note that
\[
\tp(Y| I_X I_Z X W Z) = \tp(Y| I_X I_Z W Z) = \tp(Y| I_Z W Z)
\]
by deploying the conditional independences given. 
\end{proof}

Recall that checking whether the antecedents of the rules above hold can be done as shown in Section \ref{sec:membership}, since we assume to have access to the elementary representation of $\m{M}$. The antecedent of rule 1 should be read as, given that $Z$ operates under its interventional regime and $V \setminus Z$ operates under its observational regime, $X$ is conditionally independent of $Y$ given $W$. The antecedent of rule 2 should be read as, given that $Z$ operates under its interventional regime and $V \setminus Z$ operates under its observational regime, the conditional probability distribution of $Y$ given $X W Z$ is the same in the observational and interventional regimes of $X$ and, thus, it can be transferred across regimes. The antecedent of rule 3 should be read similarly.

Clearly, if repeated application of rules 1-3 reduces a causal effect to an expression involving only observed quantities, then it is identifiable. The following theorem shows that finding the sequence of rules 1-3 to apply can be systematized in some cases. The theorem likens \cite[Theorems 3.3.2, 3.3.4 and 4.3.1, and Section 4.3.3]{Pearl2009}.\footnote{The best way to appreciate the likeness between our and Pearl's theorems is by first adding the edge $F_j \ra j$ to the causal graphs in Pearl's theorems for all $j \in V$ and, then, using {\em d}-separation to compare the conditions in our theorem and the conditional independences used in the proofs of Pearl's theorems. We omit the details because our results do not build on Pearl's, i.e. they are self-contained.}

\begin{theorem}\label{the:sufficiency}
Let $X$, $Y$ and $W$ denote three disjoint subsets of $V$. Then, $\tp(Y | I_X X W)$ is identifiable if one of the following cases applies:
\begin{itemize}
\item Case 1 (back-door criterion). If there exists a set $Z \subseteq V \setminus X Y W$ such that the following conditions hold conjunctively:
\begin{itemize}
\item Condition 1.1. $Y \tci_{\m{M}} F_X | X W Z$

\item Condition 1.2. $Z \tci_{\m{M}} X | I_X W$ and $Z \tci_{\m{M}} F_X | W$
\end{itemize}
then $\tp(Y | I_X X W) = \sum_Z \tp(Y | X W Z) \tp(Z | W)$.

\item Case 2 (front-door criterion). If there exists a set $Z \subseteq V \setminus X Y W$ such that the following conditions hold conjunctively:
\begin{itemize}
\item Condition 2.1. $Z \tci_{\m{M}} F_X | X W$

\item Condition 2.2. $Y \tci_{\m{M}} F_Z | X W Z$

\item Condition 2.3. $X \tci_{\m{M}} Z | I_Z W$ and $X \tci_{\m{M}} F_Z | W$

\item Condition 2.4. $Y \tci_{\m{M}} F_Z | I_X X W Z$

\item Condition 2.5. $Y \tci_{\m{M}} X | I_X I_Z W Z$ and $Y \tci_{\m{M}} F_X | I_Z W Z$
\end{itemize}
then $\tp(Y | I_X X W) = \sum_Z \tp(Z | X W) \sum_X \tp(Y | X W Z) \tp(X | W)$.

\item Case 3. If there exists a set $Z \subseteq V \setminus X Y W$ such that the following conditions hold conjunctively:
\begin{itemize}
\item Condition 3.1. $\tp(Z | I_X X W)$ is identifiable 

\item Condition 3.2. $Y \tci_{\m{M}} F_X | X W Z$
\end{itemize}
then $\tp(Y | I_X X W) = \sum_Z \tp(Y | X W Z) \tp(Z | I_X X W)$.

\item Case 4. If there exists a set $Z \subseteq V \setminus X Y W$ such that the following conditions hold conjunctively:
\begin{itemize}
\item Condition 4.1. $\tp(Y | I_X X W Z)$ is identifiable 

\item Condition 4.2. $Z \tci_{\m{M}} X | I_X W$ and $Z \tci_{\m{M}} F_X | W$
\end{itemize}
then $\tp(Y | I_X X W) = \sum_Z \tp(Y | I_X X W Z) \tp(Z | W)$.
\end{itemize}
\end{theorem}

\begin{proof}
To prove case 1, note that
\[
\tp(Y | I_X X W) = \sum_Z \tp(Y | I_X X W Z) \tp(Z | I_X X W) = \sum_Z \tp(Y | X W Z) \tp(Z | I_X X W)
\]
\[
= \sum_Z \tp(Y | X W Z) \tp(Z | W)
\]
where the second equality is due to rule 2 and condition 1.1, and the third due to rule 3 and condition 1.2.

To prove case 2, note that condition 2.1 enables us to apply case 1 replacing $X$, $Y$, $W$ and $Z$ with $X$, $Z$, $W$ and $\emptyset$. Then, $\tp(Z | I_X X W) = \tp(Z | X W)$. Likewise, conditions 2.2 and 2.3 enable us to apply case 1 replacing $X$, $Y$, $W$ and $Z$ with $Z$, $Y$, $W$ and $X$. Then, $\tp(Y | I_Z W Z) = \sum_X \tp(Y | X W Z) \tp(X | W)$. Finally, note that
\[
\tp(Y | I_X X W) = \sum_Z \tp(Y | I_X X W Z) \tp(Z | I_X X W) = \sum_Z \tp(Y | I_X I_Z X W Z) \tp(Z | I_X X W)
\]
\[
= \sum_Z \tp(Y | I_Z W Z) \tp(Z | I_X X W)
\]
where the second equality is due to rule 2 and condition 2.4, and the third due to rule 3 and condition 2.5. Plugging the intermediary results proven before into the last equation gives the desired result.

To prove case 3, note that
\[
\tp(Y | I_X X W) = \sum_Z \tp(Y | I_X X W Z) \tp(Z | I_X X W) = \sum_Z \tp(Y | X W Z) \tp(Z | I_X X W)
\]
where the second equality is due to rule 2 and condition 3.2.

To prove case 4, note that
\[
\tp(Y | I_X X W) = \sum_Z \tp(Y | I_X X W Z) \tp(Z | I_X X W) = \sum_Z \tp(Y | I_X X W Z) \tp(Z | W)
\]
where the second equality is due to rule 3 and condition 4.2.
\end{proof}

For instance, consider the causal graph (a) in Figure \ref{fig:example2} \cite[Figure 3.4]{Pearl2009}. Then, $\tp(y | I_x x z_3)$ can be identified by case 1 with $X=x$, $Y=y$, $W=z_3$ and $Z = z_4$ and, thus, $\tp(y | I_x x)$ can be identified by case 4 with $X=x$, $Y=y$, $W=\emptyset$ and $Z=z_3$. To see that each triplet in the conditions in cases 1 and 4 holds, we can add the edge $F_j \ra j$ to the graph for all $j \in V$ and, then, apply {\em d}-separation in the causal graph after having performed the interventions in the conditioning set of the triplet, i.e. after having removed any edge with an arrowhead into any node in the conditioning set. See \cite[3.2.3]{Pearl2009} for further details. Given the causal graph (b) in Figure \ref{fig:example2} \cite[Figure 4.1 (b)]{Pearl2009}, $\tp(z_2 | I_x x)$ can be identified by case 2 with $X=x$, $Y=z_2$, $W=\emptyset$ and $Z=z_1$ and, thus, $\tp(y | I_x x)$ can be identified by case 3 with $X=x$, $Y=y$, $W=\emptyset$ and $Z=z_2$. Note that we do not need to know the causal graphs nor their existence to identify the causal effects. It suffices to know the conditional independences in the conditions of the cases in the theorem above. Recall again that checking these can be done as shown in Section \ref{sec:membership}. The theorem above can be seen as a recursive procedure for causal effect identification: Cases 1 and 2 are the base cases, and cases 3 and 4 are the recursive ones. In applying this procedure, efficiency may be an issue, though: Finding $Z$ seems to require an exhaustive search.

\begin{figure}
\centering
\begin{tabular}{|c|c|c|}
\hline
(a)&(b)&(c)\\
\begin{tikzpicture}[inner sep=1mm]
\node at (0,0) (X) {$x$};
\node at (2,0) (Y) {$y$};
\node at (1,0) (Z6) {$z_6$};
\node at (1,1) (Z4) {$z_4$};
\node at (0,1) (Z3) {$z_3$};
\node at (2,1) (Z5) {$z_5$};
\node at (0,2) (Z1) {$z_1$};
\node at (2,2) (Z2) {$z_2$};
\path[->] (Z4) edge (X);
\path[->] (Z4) edge (Y);
\path[->] (X) edge (Z6);
\path[->] (Z6) edge (Y);
\path[->] (Z3) edge (X);
\path[->] (Z5) edge (Y);
\path[->] (Z1) edge (Z3);
\path[->] (Z2) edge (Z5);
\path[->] (Z1) edge (Z4);
\path[->] (Z2) edge (Z4);
\end{tikzpicture}
&
\begin{tikzpicture}[inner sep=1mm]
\node at (0,0) (A) {$x$};
\node at (2,-2) (B) {$z_2$};
\node at (1,-1) (C) {$z_1$};
\node at (0,-3) (Y) {$y$};
\node at (2,-0.5) (U) {$u$};
\path[dashed,->] (U) edge (A);
\path[dashed,->] (U) edge (B);
\path[->] (A) edge (C);
\path[->] (C) edge (B);
\path[->] (A) edge (Y);
\path[->] (B) edge (Y);
\end{tikzpicture}
&
\begin{tikzpicture}[inner sep=1mm]
\node at (0,0) (A) {$x_1$};
\node at (2,-2) (B) {$x_2$};
\node at (3,-1) (C) {$z$};
\node at (1,1) (U1) {$u_1$};
\node at (5,0) (U2) {$u_2$};
\node at (2,-3) (Y) {$y$};
\path[->] (A) edge (B);
\path[->] (A) edge (C);
\path[->] (C) edge (B);
\path[->] (A) edge (Y);
\path[->] (B) edge (Y);
\path[dashed,->] (U1) edge (A);
\path[dashed,->] (U1) edge (C);
\path[dashed,->] (A) edge (U2);
\path[dashed,->] (U2) edge (C);
\path[dashed,->] (U2) edge (Y);
\end{tikzpicture}\\
\hline
\end{tabular}\caption{Causal graphs in the examples. All the nodes are observed except $u$, $u_1$ and $u_2$.}\label{fig:example2}
\end{figure}
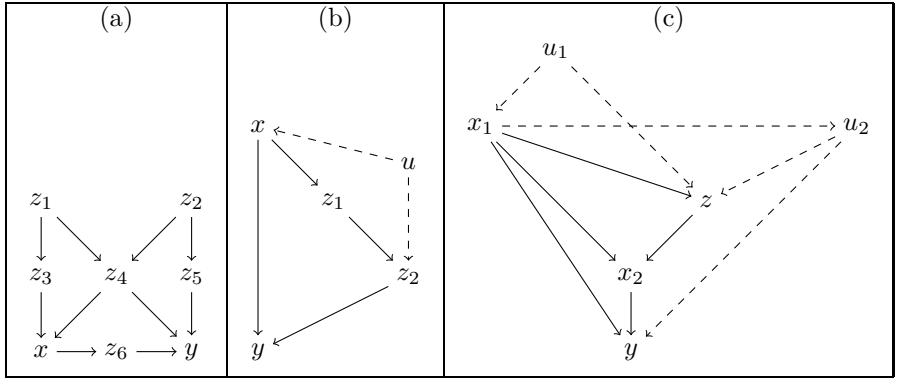

\subsubsection{Plan Evaluation}

This section covers an additional case where causal effect identification is possible. It likens \cite[Theorem 4.4.1]{Pearl2009}. See also \cite[Section 11.3.7]{Pearl2009}. Specifically, it addresses the evaluation of a plan, where a plan is a sequence of interventions. For instance, we may want to evaluate the effect on the patient's health of some treatments administered at different time points. More formally, let $X_1, \ldots, X_n$ denote the random variables on which we intervene. Let $Y$ denote the set of target random variables. Assume that we intervene on $X_k$ only after having intervened on $X_1, \ldots, X_{k-1}$ for all $1 \leq k \leq n$, and that $Y$ is observed only after having intervened on $X_1, \ldots, X_n$. The goal is to identify $\tp(Y | I_{X_1} \ldots I_{X_n} X_1 \ldots X_n)$. Let $N_1, \ldots, N_n$ denote some observed random variables besides $X_1, \ldots, X_n$ and $Y$. Assume that $N_k$ is observed before intervening on $X_k$ for all $1 \leq k \leq n$. Then, it seems natural to assume for all $1 \leq k \leq n$ and all $Z_k \subseteq N_k$ that $Z_k$ does not get affected by future interventions, i.e.
\begin{equation}\label{eq:natural}
Z_k \tci_{\m{M}} X_k \ldots X_n | I_{X_k} \ldots I_{X_n} X_1 \ldots X_{k-1} Z_1 \ldots Z_{k-1}
\end{equation}
and
\begin{equation}\label{eq:natural2}
Z_k \tci_{\m{M}} F_{X_k} \ldots F_{X_n} | X_1 \ldots X_{k-1} Z_1 \ldots Z_{k-1}.
\end{equation}

\begin{theorem}\label{the:plan}
If there exist disjoint sets $Z_k \subseteq N_k$ for all $1 \leq k \leq n$ such that
\begin{equation}\label{eq:plan}
Y \tci_{\m{M}} F_{X_k} | I_{X_{k+1}} \ldots I_{X_n} X_1 \ldots X_n Z_1 \ldots Z_k
\end{equation}
then $\tp(Y | I_{X_1} \ldots I_{X_n} X_1 \ldots X_n) =$
\[
\sum_{Z_1 \ldots Z_n} \tp(Y | X_1 \ldots X_n Z_1 \ldots Z_n) \prod_{k=1}^n \tp(Z_k | X_1 \ldots X_{k-1} Z_1 \ldots Z_{k-1}).
\]
\end{theorem}

\begin{proof}
Note that
\[
\tp(Y | I_{X_1} \ldots I_{X_n} X_1 \ldots X_n)
\]
\[
= \sum_{Z_1} \tp(Y | I_{X_1} \ldots I_{X_n} X_1 \ldots X_n Z_1) \tp(Z_1 | I_{X_1} \ldots I_{X_n} X_1 \ldots X_n)
\]
\[
= \sum_{Z_1} \tp(Y | I_{X_2} \ldots I_{X_n} X_1 \ldots X_n Z_1) \tp(Z_1 | I_{X_1} \ldots I_{X_n} X_1 \ldots X_n)
\]
\[
= \sum_{Z_1} \tp(Y | I_{X_2} \ldots I_{X_n} X_1 \ldots X_n Z_1) \tp(Z_1)
\]
where the second equality is due to rule 2 and Equation (\ref{eq:plan}), and the third due to rule 3 and Equations (\ref{eq:natural}) and (\ref{eq:natural2}). For the same reasons, we have that
\[
\tp(Y | I_{X_1} \ldots I_{X_n} X_1 \ldots X_n)
\]
\[
= \sum_{Z_1 Z_2} \tp(Y | I_{X_2} \ldots I_{X_n} X_1 \ldots X_n Z_1 Z_2) \tp(Z_1) \tp(Z_2 | I_{X_2} \ldots I_{X_n} X_1 \ldots X_n Z_1)
\]
\[
= \sum_{Z_1 Z_2} \tp(Y | I_{X_3} \ldots I_{X_n} X_1 \ldots X_n Z_1 Z_2) \tp(Z_1) \tp(Z_2 | X_1 Z_1).
\]
Continuing with this process for $Z_3, \ldots, Z_n$ yields the desired result.
\end{proof}

For instance, consider the causal graph (c) in Figure \ref{fig:example2} \cite[Figure 4.4]{Pearl2009}. We do not need to know the graph nor its existence to identify the effect on $y$ of the plan consisting of $I_{x_1} x_1$ followed by $I_{x_2} x_2$. It suffices to know that $N_1 = \emptyset$, $N_2 = z$, $y \tci_{\m{M}} F_{x_1} | I_{x_2} x_1 x_2$, and $y \tci_{\m{M}} F_{x_2} | x_1 x_2 z$.  Recall also that $z \tci_{\m{M}} x_2 | I_{x_2} x_1$ and $z \tci_{\m{M}} F_{x_2} | x_1$ are known by Equations (\ref{eq:natural}) and (\ref{eq:natural2}). Then, the desired effect can be identified thanks to the theorem above by setting $Z_1 = \emptyset$ and $Z_2 = z$.

In applying the theorem above, efficiency may be an issue again: Finding $Z_1, \ldots, Z_n$ seems to require an exhaustive search. An effective way to carry out this search is as follows: Select $Z_k$ only after having selected $Z_1, \ldots, Z_{k-1}$, and such that $Z_k$ is a minimal subset of $N_k$ that satisfies Equation (\ref{eq:plan}). If no such subset exists or all the subsets have been tried, then backtrack and set $Z_{k-1}$ to a different minimal subset of $N_{k-1}$. We now show that this procedure finds the desired subsets whenever they exist. Assume that there exist some sets $Z_1^*, \ldots, Z_n^*$ that satisfy Equation (\ref{eq:plan}). For $k=1$ to $n$, set $Z_k$ to a minimal subset of $Z_k^*$ that satisfies Equation (\ref{eq:plan}). If no such subset exists, then set $Z_k$ to a minimal subset of $( \bigcup_{i=1}^k Z_i^* ) \setminus \bigcup_{i=1}^{k-1} Z_i$ that satisfies Equation (\ref{eq:plan}). Such a subset exists because setting $Z_k$ to $( \bigcup_{i=1}^k Z_i^* ) \setminus \bigcup_{i=1}^{k-1} Z_i$ satisfies Equation (\ref{eq:plan}), since this makes $Z_1 \ldots Z_k = Z_1^* \ldots Z_k^*$. In either case, note that $Z_k \subseteq N_k$. Then, the procedure outlined will find the desired subsets.

We can extend the previous theorem to evaluate the effect of a plan on the target random variables $Y$ and on some observed non-control random variables $W \subseteq N_n$. For instance, we may want to evaluate the effect that the treatment has on the patient's health at intermediate time points, in addition to at the end of the treatment. This scenario is addressed by the following theorem, whose proof is similar to that of the previous theorem. The theorem likens \cite[Theorem 4]{PearlandRobins1995}.

\begin{theorem}\label{the:plan2}
If there exist disjoint sets $Z_k \subseteq N_k \setminus W$ for all $1 \leq k \leq n$ such that
\[
W Y \tci_{\m{M}} F_{X_k} | I_{X_{k+1}} \ldots I_{X_n} X_1 \ldots X_n Z_1 \ldots Z_k
\]
then $\tp(W Y | I_{X_1} \ldots I_{X_n} X_1 \ldots X_n) =$
\[
\sum_{Z_1 \ldots Z_n} \tp(W Y | X_1 \ldots X_n Z_1 \ldots Z_n) \prod_{k=1}^n \tp(Z_k | X_1 \ldots X_{k-1} Z_1 \ldots Z_{k-1}).
\]
\end{theorem}

Finally, note that in the previous theorem $X_k$ may be a function of $X_1 \ldots X_{k-1}$ $W_1 \ldots W_{k-1} Z_1 \ldots Z_{k-1}$, where $W_k = ( W \setminus \bigcup_{i=1}^{k-1} W_i ) \cap N_k$ for all $1 \leq k \leq n$. For instance, the treatment prescribed at any point in time may depend on the treatments prescribed previously and on the patient's response to them. In such a case, the plan is called conditional, otherwise is called unconditional. We can evaluate alternative conditional plans by applying the theorem above for each of them. See also \cite[Section 11.4.1]{Pearl2009}.

\subsection{Context-specific Independences Revisited}\label{sec:csirevisited}

As mentioned in Section \ref{sec:csi}, we can extend the results in this paper to independence models containing context-specific independences of the form $I \ci J | K, L=l$ by just rephrasing the properties CI0-3 and ci0-3 to accommodate them. In the causal setup described above, for instance, we may want to represent triplets with interventions in their third element as long as they do not affect the first two elements of the triplets, i.e. $I \tci J | K M I_M I_N$ with $I$, $J$, $K$, $M$ and $N$ disjoint subsets of $V$, which should be read as follows: Given that $M N$ operates under its interventional regime and $V \setminus M N$ operates under its observational regime, $I$ is conditionally independent of $J$ given $K M$. Note that an intervention is made on $N$ but the resulting value is not considered in the triplet, e.g. we know that a treatment has been prescribed but we ignore which. The properties CI0-3 can be extended to these triplets by simply adding $M I_M I_N$ to the third member of the triplets. That is, let $C = M I_M I_N$. Then:

\begin{itemize}

\item[(CI0)] $I \tci J | K C \Leftrightarrow J \tci I | K C$.

\item[(CI1)] $I \tci J | K L C, I \tci K | L C \Leftrightarrow I \tci J K | L C$.

\item[(CI2)] $I \tci J | K L C, I \tci K | J L C \Rightarrow I \tci J | L C, I \tci K | L C$.

\item[(CI3)] $I \tci J | K L C, I \tci K | J L C \Leftarrow I \tci J | L C, I \tci K | L C$.

\end{itemize}
Similarly for ci0-3.

Another case that we may want to consider is when a triplet includes interventions in its third element that affect its second element, i.e. $I \tci J | K M I_J I_M I_N$ with $I$, $J$, $K$, $M$ and $N$ disjoint subsets of $V$, which should be read as follows: Given that $J M N$ operates under its interventional regime and $V \setminus J M N$ operates under its observational regime, the causal effect on $I$ is independent of $J$ given $K M$. These triplets liken the probabilistic causal irrelevances in \cite[Definition 7]{GallesandPearl1997}. The properties CI1-3 can be extended to these triplets by simply adding $M I_J I_K I_M I_N$ to the third member of the triplets. Note that CI0 does not make sense now, i.e. $I$ is observed whereas $J$ is intervened on. Let $C = M I_J I_K I_M I_N$. Then:

\begin{itemize}

\item[(CI1)] $I \tci J | K L C, I \tci K | L C \Leftrightarrow I \tci J K | L C$.

\item[(CI2)] $I \tci J | K L C, I \tci K | J L C \Rightarrow I \tci J | L C, I \tci K | L C$.

\item[(CI3)] $I \tci J | K L C, I \tci K | J L C \Leftarrow I \tci J | L C, I \tci K | L C$.

\item[(CI1')] $I \tci J | I' L C, I' \tci J | L C \Leftrightarrow I I' \tci J | L C$.

\item[(CI2')] $I \tci J | I' L C, I' \tci J | I L C \Rightarrow I \tci J | L C, I' \tci J | L C$.

\item[(CI3')] $I \tci J | I' L C, I' \tci J | I L C \Leftarrow I \tci J | L C, I' \tci J | L C$.

\end{itemize}
Similarly for ci1-3.

\section{Discussion}\label{sec:discussion}

In this work, we have proposed to represent semigraphoids, graphoids and compositional graphoids by their elementary triplets. We have also shown how this representation helps performing some operations with independence models, including causal reasoning. For this purpose, we have rephrased in terms of conditional independences some of Pearl's results for causal effect identification. We find interesting to explore non-graphical approaches to causal reasoning in the vein of \cite{Dawid2015}, because of the risks of relying on causal graphs for causal reasoning, e.g. a causal graph of the domain at hand may not exist and/or the effects of an intervention may not be local. See \cite{Dawid2010a,Dawid2010b} for a detailed account of these risks. Pearl also acknowledges the need to develop non-graphical approaches to causal reasoning \cite[p. 10]{GallesandPearl1997}. As future work, we consider seeking for necessary conditions for non-graphical causal effect identification (recall that the ones described in this paper are just sufficient). We also consider implementing and experimentally evaluating the efficiency of some of the operations discussed in this work, including a comparison with their counterparts in the dominant triplet representation as reported in \cite{BaiolettiBV09,Baioletti20112,BaiolettiPV13}.

\section*{Acknowledgments}

We would like to thank the anonymous Reviewers for their comments, which helped us to improve the original manuscript substantially.

\bibliographystyle{abbrv}
\bibliography{Pena_paper7revision2}

\end{document}